\documentclass[10pt,journal,cspaper,compsoc]{IEEEtran}

\usepackage[margin=10pt,font=footnotesize,labelfont=bf,labelsep=endash]{caption}

\usepackage{float}
\usepackage{subcaption}

\usepackage{algorithmic}
\usepackage[linesnumbered, algoruled, vlined]{algorithm2e}

\usepackage{graphicx}
\usepackage{amssymb}
\usepackage{amsmath}
\usepackage{amsfonts}
\usepackage{color}
\usepackage{dsfont}
\usepackage{upgreek}
\usepackage{bbm}
\usepackage{url}
\usepackage{multirow}
\usepackage{xspace}
\usepackage{eucal}
\usepackage{proof}

\usepackage{hyperref}

 \usepackage[table]{xcolor}
 \usepackage{arydshln}

\newtheorem{proposition}{\bf Proposition}
\def\relationindicator{\delta(y_{ij} \neq 0)}

\newcommand{\bY}{\mathbf{Y}}

\newcommand{\cX}{\mathcal{X}}

\newcommand{\cV}{\mathcal{V}}
\newcommand{\cU}{\mathcal{U}}

\def\exp{\operatorname*{exp\,}}

\def\sign{\operatorname*{sign\,}}

\def\x{{\boldsymbol x}}
\def\w{{\boldsymbol w}}

\def\z{{\boldsymbol z}}

\def\Z{{\mathbf  Z}}
\def\Y{{\mathbf  Y}}
\def\A{{\mathbf  A}}

\def\bY{{\mathbf  Y}}

\newcommand{\comment}[1]{}

\def\T{{\!\top}}
\def\Real{\mathbb{R}}
\newcommand{\st}{{{\rm s.t.}\;\;}\xspace}

\def\x{{\boldsymbol x}}
\def\w{{\boldsymbol w}}

\def\hs{\Phi}
\def\bh{\hs}

\def\dhamm{d_{\mathrm{H}}}
\def\shamm{s_{\mathrm{H}}}

\def\best{\bf \cellcolor[gray]{0.998}}

\newcommand{\figcenter}[1]{\raisebox{-0.5\height}{#1}}
\newcommand{\colseperator}{\hspace*{.1in}}
\newcommand{\rowseperator}{\vspace*{.05in}}

\def\fasth{{FastHash}\xspace}
\def\ssum{\sum}
\def\tree{{T}}
\def\block{{\mathcal{B}}}

\begin{document}

\title{Supervised Hashing Using Graph Cuts and Boosted Decision Trees}

\author{Guosheng Lin, Chunhua Shen, Anton {van den Hengel}
\IEEEcompsocitemizethanks{
\IEEEcompsocthanksitem
Authors are with  the School of Computer Science, The University of Adelaide,
Australia; and
Australian Research Council Centre of Excellence for Robotic Vision;
\protect\\
Corresponding author: Chunhua Shen (chunhua.shen@adelaide.edu.au);
\protect\\
\textcolor{blue}{
(c) IEEE 2015. Appearing in  IEEE Trans. Pattern Analysis and Machine Intelligence.
Content may change prior to final publication.
}
}
\thanks{}}

\markboth{Appearing in IEEE Trans. Pattern Analysis and Machine Intelligence, Feb.\ 2015}{}

\IEEEcompsoctitleabstractindextext{%
\begin{abstract}

To build large-scale query-by-example image retrieval systems, embedding image features into a
binary Hamming space provides great benefits.
Supervised hashing aims to map the original features to compact binary codes that are
able to preserve label based similarity in the binary Hamming space.
Most existing approaches apply a single form of hash function,
and an optimization process which is typically deeply coupled to this specific form.
This tight coupling restricts the flexibility of those methods,
and can result in complex optimization problems that are difficult to solve.
In this work we proffer a flexible yet simple framework that is able to accommodate
different types of loss functions and hash functions.
The proposed framework allows a number of existing approaches to hashing to be placed in context,
and simplifies the development of new problem-specific hashing methods.
Our framework decomposes the hashing learning problem into two steps:
binary code (hash bit) learning and hash function learning.
The first step can typically be formulated as binary quadratic problems,
and the second step can be accomplished by training a standard binary classifier.
For solving large-scale binary code inference,
we show how it is possible to ensure that the binary quadratic problems are submodular such that efficient graph cut methods may be used.
To achieve efficiency as well as efficacy on large-scale high-dimensional data,
we propose to use boosted decision trees as the hash functions,
which are nonlinear, highly descriptive, and are very fast to train and evaluate.
Experiments demonstrate that the proposed method significantly outperforms most state-of-the-art methods,
especially on high-dimensional data.

\end{abstract}


}

\maketitle

\section{Introduction}

An explosion in the size of the datasets has been witnessed in the past a few years.
It becomes more and more demanding to cope with image datasets with tens of millions of images, in terms of both
efficient storage and processing.
Hashing methods construct a set of hash functions that map the original features into
compact binary codes. Hashing enables fast nearest neighbor search by using look-up tables or Hamming distance based ranking.
Moreover, compact binary codes are extremely efficient for large-scale data storage.
Example applications include image retrieval (\cite{Torralba08,wang2010semi}), image matching \cite{Strecha2012},
object detection \cite{fastdection}, etc.

Loss functions for learning-based hashing
are typically defined on the basis of the Hamming distance or Hamming affinity  of similar and
dissimilar data pairs. Hamming affinity is calculated by the inner
product of two binary codes (a binary code takes a value from
$\{-1,1\}$).
Existing methods thus tend to optimize a single form of hash function.
The common forms of hash functions
are linear perceptron functions (e.g.,
Minimal Loss Hashing (MLH) \cite{norouzi2011minimal},
Semi-supervised Hashing (SPLH) \cite{wang2010semi},
Iterative Quantization (ITQ) \cite{gong2012iterative},
Locality-Sensitive Hashing (LSH) \cite{Gionis1999}),
kernel functions
(Supervised Hashing with Kernels (KSH) \cite{KSH}),
eigenfunctions (Spectral Hashing (SPH) \cite{weiss2008spectral},
Multidimensional Spectral Hashing (MDSH) \cite{MDSH}).
The optimization procedure is then coupled with the selected family of hash
functions.
Different types of hash functions offer a trade-off between testing time and
fitting capacity. For example, compared with kernel functions, the
simple linear perceptron function is usually much more efficient for
evaluation but can have a relatively low accuracy for nearest neighbor search.
This coupling often results in a highly
non-convex optimization problem which can be very challenging to optimize.
As an example, the loss functions in MDSH, KSH and
Binary Reconstructive Embeddings (BRE) \cite{kulis2009learning} all take a
similar form that aims to minimize the difference between the Hamming
affinity (or distance) of data pairs and the ground truth.
However, the optimization procedures used are
coupled with the form of hash functions (eigenfunctions, kernel
functions) and thus different optimization techniques are needed for each.

Our framework, however, is able to
accommodate any loss function defined on the Hamming distance/affinity
of data pairs, such as the loss functions used in KSH, BRE or MLH.
We decompose the learning into two steps:
the binary codes inference step and the hash function learning step.
We can formulate the optimization problem of any Hamming distance/affinity based loss as binary quadratic problems,
hence different types of loss functions are unified into
the same optimization problem, which significantly simplifies the optimization.
With this decomposition the hash function learning
becomes a binary classification problem,
hence we can learn various types of hash function,
like perceptrons, kernel and decision tree hash functions,
by simply training binary classifiers.

Many supervised hashing approaches
require complex optimization for directly learning hash functions,
and hence may only be tractable for small scale training data.
In our approach, we propose an efficient graph cut based block search algorithm
for solving the large-scale binary code inference problem,
thus our method can be easily trained on large-scale datasets.

Recent advances in the feature learning (\cite{coates2011importance,krizhevsky2012imagenet}) show that
high-dimensional features are essential for achieving good performance.
For example, the dimension of codebook based features is usually in the tens of thousands.
Many existing hashing methods become impractically slow when trained on large scale high-dimensional features.
Non-linear hash functions, e.g., the kernel hash function employed in KSH,
have shown much improved performance over the linear hash function. %
However, kernel functions can be extremely expensive to evaluate for both training and testing on high-dimensional features.
Here we propose to learn decision trees as hash functions for non-linear mapping.
Decision trees only involve simple comparison operations,
thus they are very efficient to evaluate.
Moreover, decision trees are able to work on quantized data without significant performance loss,
and hence only consume a small amount of memory for training.

The main contributions of this work are as follows.
\begin{itemize}

\item
    We propose a flexible and efficient hashing framework
   which is able to incorporate various kinds of loss functions and hash functions.

      We decompose the learning
      procedure into two steps: binary code inference, and hash function learning.
      This decomposition
      simplifies the hash function
      learning problem into a standard binary classification problem.
      An arbitrary classifier, such as linear or kernel Support Vector Machines (SVM), boosting, or neural networks,
      may thus be adopted to learn the hash functions.

      We are able to incorporate various types of loss function in a unified manner.
      We show that any type of loss function
      (e.g., the loss functions in KSH, BRE, MLH)
      defined on Hamming affinity or Hamming distance,
      can be equivalently converted into a standard quadratic function,
      thus we can solve a standard binary quadratic problem for binary code inference.

\item

    For binary code inference, we propose sub-modular formulations and an
    efficient graph cut \cite{boykov2001fast} based block search method for solving large-scale binary code inference.

\item

   We propose to use (ensembles of) decision trees as hash functions for supervised hashing,
   which can easily deal with a very large number of high-dimensional training data
   and has the desired non-linear mapping.

\item

   Our method significantly outperforms many existing methods in terms of retrieval accuracy.
   For high-dimensional data, our method is also  orders of magnitude faster for training.

\end{itemize}

We made the code available at \url{https://bitbucket.org/chhshen/fasthash/}.

\subsection{Related work}

Hashing methods aim to preserve some notion of
similarity (or distance) in the Hamming space.
These methods can be roughly categorized as being either supervised or unsupervised.
Unsupervised hashing methods
(\cite{Gionis1999,weiss2008spectral,MDSH,gong2012iterative,KMH,liu2011hashingGraphs,jae2012,CVPR13aShen})
try to preserve the similarity which is often calculated in the original feature space.
For example,
LSH
\cite{Gionis1999} generates random linear hash functions to
approximate cosine similarity;
SPH (\cite{weiss2008spectral,MDSH}) learns eigenfunctions that preserve Gaussian affinity;
ITQ \cite{gong2012iterative} approximates the Euclidean distance in the Hamming space.
Supervised hashing is designed to preserve the label-based similarity
(\cite{SalakhutdinovH07,kulis2009learning,norouzi2011minimal,zhangSTHs,wang2010semi,KSH,li2013learning,lin2014optimizing}).
This might take place, for example, in the case where
images from the same category are defined as being semantically similar to each other.
Supervised hashing has received increasing attention recently
(e.g., KSH \cite{KSH},
BRE \cite{kulis2009learning}).
Our method targets supervised hashing.
Preliminary results of our work appeared in \cite{TSH} and \cite{fasthash}.
Various optimization techniques are proposed in existing methods.
For example, random projection is
used in LSH and Kernelized Locality-Sensitive Hashing (KLSH) \cite{KLSH};
spectral graph analysis for
exploring the data manifold is used in SPH \cite{weiss2008spectral}, MDSH \cite{MDSH},
STH \cite{zhang2010self}, Hashing with Graphs (AGH) \cite{liu2011hashingGraphs}, and inductive hashing \cite{CVPR13aShen};
vector quantization is used in ITQ \cite{gong2012iterative},
and K-means Hashing \cite{KMH};
kernel methods are used in KSH \cite{KSH} and KLSH \cite{KLSH}.
MLH \cite{norouzi2011minimal} optimizes a hinge-like loss.
The optimization techniques in most existing work are tightly coupled
with their loss functions and hash functions.
In contrast, our method breaks this coupling and easily incorporates various
types of loss function and hash function.

A number of existing hash methods have explicitly or implicitly employed two-step optimization based strategies
for hash function learning,
like Self-Taught Hashing (STH) \cite{zhang2010self}, MLH \cite{norouzi2011minimal},
Hamming distance metric learning \cite{norouzi2012hamming}, ITQ \cite{gong2012iterative}
and angular quantization based binary code learning \cite{gong2012angular}.
However, in these existing methods,
the optimization techniques for binary inference and hash function
learning are deeply coupled to their specific form of loss function and hash functions,
and none of them is as general as our learning framework.

STH \cite{zhang2010self} explicitly
employs a two-step learning scheme for optimizing the Laplacian affinity loss.
The Laplacian affinity loss in STH only tries to pull
together similar data pairs but does not push away dissimilar data pairs,
which may lead to inferior performance \cite{EE}.
Moreover, STH employs a spectral method for binary code inference,
which usually leads to inferior binary solutions due to its loose relaxation.
Moreover, the spectral method does not scale well on large training data.
In contrast, we are able to incorporate any hamming distance or affinity based loss function,
and propose an efficient graph cut based method for large scale binary code inference.

MLH \cite{norouzi2011minimal} learns hash functions by optimizing a convex-concave upper-bound of a hinge loss function (or BRE loss function).
They need to solve a binary code inference problem during optimization, for which they propose a so-called loss-adjusted inference algorithm. A similar technique is also applied in \cite{norouzi2012hamming}.
The training of ITQ \cite{gong2012iterative} also involves a two-step optimization strategy.
ITQ iteratively generates the binary code and learns a rotation matrix by minimizing the quantization error against the binary code. They generate the binary code simply by thresholding.

The problem of similarity search on high-dimensional data is also addressed in \cite{li2011learning}.
Their method extends vocabulary tree based search methods (\cite{nister2006scalable,philbin2007object})
by replacing vocabulary trees with boosted trees.
This type of search method represents the image as the evidence of a large number of visual words,
which are vectors with thousands or even millions dimensions.
Then this visual word representation is fed into an
inverted index based search algorithm to output the final retrieval result.
Clearly hashing methods are different from these inverted index based search methods.
Our method is in the vein of supervised hashing methods:
mapping data points into binary codes so that the hamming distance on binary
codes reflects the label based similarity.

\section{Flexible Two-Step Hashing}

Let $\cX=\{\x_1, ..., \x_n\} \subset \Real^d$ denote a set of training points.
Label based similarity information is described by an affinity matrix: $\bY$,
which is the ground truth for supervised learning.
The element in $\bY$: $y_{ij}$ indicates the similarity between two data points $\x_i$ and $\x_j$; and $y_{ij}=y_{ji}$.
Specifically, $y_{ij}=1$ if two data points are similar (relevant), $y_{ij}=-1$ if dissimilar (irrelevant)
and $y_{ij}=0$ if the pairwise relation is undefined.
We aim to learn a set of hash functions which preserve the label based similarity in the Hamming space.
A hash function is denoted by $h(\cdot)$ with binary output: $h(\x) \in \{-1, 1\}$.
The output of $m$ hash functions is denoted by $\bh(\x)$:
 \begin{align}
	\bh(\x)=[h_1(\x), h_2(\x), \dots, h_m(\x)],
 \end{align}
which is a $m$-bit binary vector: $\bh(\x) \in \{-1,1\}^m$.

The loss function in hashing learning for preserving pairwise similarity relations is typically defined in terms of the Hamming distance or Hamming affinity of data pairs.
The Hamming distance between two binary codes is the number of bits taking different values:
\begin{align}
	\label{fg-eq:indicator}
	\dhamm(\x_i, \x_j)=\sum_{r=1}^m \delta(h_r(\x_i) \neq h_r(\x_j)).
\end{align}
Here $\delta(\cdot) \in \{0, 1\}$ is an indicator function which
outputs $1$ if the input is true and $0$ otherwise.
Generally, the formulation of hashing learning encourages small Hamming distances for similar data pairs and large distances for dissimilar data pairs.
Closely related to Hamming distance, the Hamming affinity is calculated by the inner product of two binary codes:
\begin{align}
	\shamm(\x_i, \x_j)=\sum_{r=1}^mh_r(\x_i)h_r(\x_j).
\end{align}
As shown in \cite{KSH},
the Hamming affinity is in one-to-one correspondence with the Hamming distance.
We solve the following optimization for hash function learning:
 \begin{align}
 	\label{tsh-eq:opt_main}
	\min_{\bh(\cdot)} \sum_{i=1}^n\sum_{j=1}^n
    \relationindicator L(\bh(\x_i), \bh(\x_j); y_{ij}).
 \end{align}
Here $\relationindicator \in \{0, 1\}$  indicates
whether the relation between two data points is defined, and
$L(\cdot)$ is a loss function that
measures how well the binary codes
match the similarity ground truth $y_{ij}$. Various types
of loss functions  $ L (\cdot)$ have been
proposed, and will be discussed in detail in the next section.
Most existing methods try to directly optimize the objective function in
\eqref{tsh-eq:opt_main} in order to learn
the parameters of hash functions (\cite{KSH,norouzi2011minimal, kulis2009learning, MDSH}).
This inevitably means that the optimization process is tightly coupled to the
form of hash functions used,
which makes it non-trivial to extend a method to use
other types of hash functions.
Moreover, this coupling usually
results in challenging optimization problems.

As an example, the KSH loss function, which is defined on Hamming affinity, is written as follows:
\begin{align}
 	\label{fh-eq:ksh_loss_org}
	L_{\mathrm{KSH}} = \sum_{i=1}^n\sum_{j=1}^n
     \relationindicator \biggr[m y_{ij} - \sum_{r=1}^m h_r(\x_i)h_r(\x_j) \biggr]^2.
\end{align}
We use $\relationindicator$ to prevent undefined pairwise relations from having an impact on the training objective.
Intuitively, this optimization encourages the Hamming affinity value of a data pair to be close to the ground truth value.
The form of hash function in KSH is the kernel function:
\begin{align}
	h(\x) = \sign\left[  \sum_{q=1}^Q w_q\kappa(\x_q', \x) + b \right],
 \end{align}
in which $\cX'=\{\x'_1,\dots,\x'_Q\}$ are $Q$ support vectors;
KSH directly solve the optimizations in \eqref{fh-eq:ksh_loss_org} for learning the hash functions.
If we prefer other forms of hash functions, the optimization of KSH would not be applicable.
For example if using the decision-tree hash function which is more suitable for high-dimensional data,  it is not clear how to learn decision trees by directly optimizing \eqref{fh-eq:ksh_loss_org}.
Moreover, KSH uses a set of predefined support vectors which are randomly sampled from the training set, and it does not have a sparse solution of the weighting parameters. Hence this unsophisticated kernel method would be impracticable for large-scale training and computationally expensive for evaluation.

Here we develop a general and flexible two-step learning framework,
which is readily to incorporate various forms of loss functions and hash functions.
Basically,  partly inspired by STH \cite{zhang2010self}, we decompose the
learning procedure into two steps: the first step for binary code
inference and the second step for hash function learning.
We introduce auxiliary variables $z_{r,i} \in \{-1, 1\}$
as the output of the $r$-th hash function on $\x_i$:
\begin{align}
	z_{r,i}=h_r(\x_i).
\end{align}
Clearly,
$z_{r,i}$ represents the $r$-th bit of the binary code of the $i$-th data point.
With these auxiliary variables, the problem in \eqref{tsh-eq:opt_main} can be decomposed into two
sub-problems:
\begin{subequations}
\label{tsh-eq:opt_step1}
\begin{align}
& \min_{\Z} \sum_{i=1}^n\sum_{j=1}^n \relationindicator
     L(\z_i, \z_j; y_{ij}), \;\;\; \\
    & \st \;\; \Z \in \{-1, 1\}^{ m \times n};
 \end{align}
\end{subequations}
and,
\begin{align}
	\min_{\bh(\cdot)} \;\; & \ssum_{r=1}^{m} \ssum_{i=1}^n \delta( z_{r,i} = h_r(\x_i) ).
    \label{tsh-eq:opt_step2_org}
\end{align}
Here $\Z$ is the matrix of $m$-bit binary codes for all $n$ training data points;
$\z_i$ is the binary code vector corresponding to $i$-th data point.
$\delta(\cdot)$ is an indicator function.
In this way, the hashing learning in \eqref{tsh-eq:opt_main} now becomes two relatively simpler tasks---solving \eqref{tsh-eq:opt_step1} (Step 1) and \eqref{tsh-eq:opt_step2_org}  (Step 2).
Clearly, Step 1 is to solve for binary codes,
and Step 2 is to solve simple binary classification problems.

We sequentially solve for one bit at a time conditioning on previous bits.
Hence we solve these two steps alternatively, rather than completely separating these two steps.
After solving for one bit, the binary codes are updated by applying the learned hash function.
Hence the learned hash function is able to influence the binary code inference for the next bit.
This bit-wise optimization strategy helps to simplify the optimization, and the error of one learned hash function can be propagated and compensated for when learning the next bit.

In the following sections, we describe how to solve these two steps for one bit.
In Sec. \ref{tsh-sec:step1}, we show that any hamming distance or affinity based loss function can be equivalently reformulated as a binary quadratic problem. For binary code inference, we propose a graph cut based block search method for efficiently solving the binary code inference (Sec. \ref{tsh-sec:step1-gc}).
Later we discuss training different types of hash functions in Sec. \ref{tsh-sec:step2}. Especially we introduce the decision tree hash functions which provide the desirable non-linear mapping and are highly efficient for evaluation (Sec. \ref{tsh-sec:step2-tree}).
With the proposed efficient binary code inference algorithm, our method is not only flexible,
but also capable of large-scale training.
We refer to our method as \fasth. The algorithm is shown in Algorithm \ref{tsh-alg:main}.

\begin{algorithm}[t]
	\caption{\small FastHash (flexible two-step hashing)}

	\label{tsh-alg:main}

	\KwIn{training data points: $\{\x_1,...\x_n\}$;
	affinity matrix: $\Y$; bit length: $m$. }
	\KwOut{hash functions: $\bh=[h_1, ..., h_m]$.}
	Initialization: construct blocks:$\{ \block_1, \block_2, ...\}$ for Block GraphCut,
	Algorithm \ref{fh-alg:block} shows an example\;
	\For{ $r=1,...,m$}
	{
		Step-1: call Algorithm \ref{fh-alg:step1} to solve the binary code inference in \eqref{tsh-eq:bqp-org},
		obtain binary codes of the $r$-th bit\;
		Step-2: solve binary classification in \eqref{tsh-eq:opt_step2_onebit} to obtain one hash function $h_r$
		(e.g., solve linear SVM in \eqref{eq:lsvm} or boosted tree learning in \eqref{eq:adaboost}) \;
		Update the binary codes of the $r$-th bit by applying the learned hash function $h_r$\;
	}

\end{algorithm}

\begin{algorithm}[t]

	\caption{\small An example for constructing blocks}

	\label{fh-alg:block}
    \KwIn{training data points: $\{\x_1,...\x_n\}$; affinity matrix: $\Y$.}
	\KwOut{blocks:$\{ \block_1, \block_2, ...\}$.}
	$\cV \leftarrow \{\x_1,...,\x_n\}$; $t=0$\;
	\Repeat{$\cV = \emptyset$}     {
		$t=t+1$; $\block_t \leftarrow \emptyset$\;
		Randomly selected $\x_i$ from $\cV$\;
		Initialize $\cU$ as the joint set of $\cV$ and similar examples of $\x_i$ \;
		\For{each $\x_j$ in $\cU$}{
			\If{$\x_j$ is not dissimilar with any examples in $\block_t$}
			{ add $\x_j$ to $\block_t$ \;
			remove $\x_j$ from $\cV$ \;}
		}
	}

\end{algorithm}

\begin{algorithm}[t]

	\caption{\small Step-1: Block GraphCut for binary code inference}
		\label{fh-alg:step1}
		\KwIn{affinity matrix: $\Y$; bit length: $r$; blocks:$\{ \block_1, \block_2, ...\}$;
			   binary codes: $\{\z_1, ..., \z_{r-1} \}$.
		}
		\KwOut{binary codes of one bit: $\z_r$.}
		\Repeat{max iteration is reached}
		{
			Randomly permute all blocks\;
			\For{each $\block_i$}
			{
				Solve the inference in \eqref{fh-eq:opt_step1_block} on $\block_i$ using graph cuts\;
			}
		}

\end{algorithm}

\section{Step 1: binary code inference}

\label{tsh-sec:step1}

When solving for the $r$-th bit, the binary codes of the previous $(r-1)$ bits are fixed,
and the bit length $m$ is set to $r$.
The binary code inference problem is:
 \begin{align}
 	\label{tsh-eq:opt_step1-onebit}
	\min_{\z_{(r)}  \in \{-1, 1\}^n }
    \sum_{i=1}^n\sum_{j=1}^n
    \relationindicator l_r(z_{r,i}, z_{r,j}; y_{ij}).
 \end{align}
Here $\z_{(r)}$ is the $n$-dimensional binary code vector we seek.
It represents the binary hash codes of the $n$ training data points for the $r$-th bit.
$z_{r,i}$ is the binary code of the $i$-th data point and the $r$-th bit.
$l_r$ represents the loss function output for the $r$-th bit, conditioning on previous bits:
\begin{align}
  l_r(z_{r,i}, z_{r,j}; y_{ij})=L(z_{r,i}, z_{r,j} ; \z_{i}^{(r-1)}, \z_{j}^{(r-1)} y_{ij}).
\end{align}
Here $\z_{i}^{(r-1)}$ is the binary code vector of the $i$-th data point in all previous $(r-1)$ bits.

Based on the following proposition, we are able to
rewrite the binary code inference problem with
any Hamming affinity (or distance) based loss function $L(\cdot)$
into a standard quadratic problem.
\begin{proposition}
\label{tsh-pro:p1}
 For any loss function $l(z_1,z_2)$ that is defined on a pair of binary input variables
 $z_1,z_2 \in \{-1,1\}$ and
 $l(1,1)=l(-1,-1)$, $l(1,-1)=l(-1,1)$,
we can define a quadratic function $g(z_1,z_2)$ that is equal to $l(z_1,z_2)$. %
We have following equations:
 \begin{align}
	l(z_1,z_2)& = \frac{1}{2} \biggr[ z_1 z_2 (l^{(11)} - l^{(-11)}) + l^{(11)} + l^{(-11)} \biggr], \notag \\
	& = \frac{1}{2} z_1 z_2 (l^{(11)} - l^{(-11)}) + {\rm const.} \notag \\
	& = g(z_1,z_2).
  \end{align}
Here $l^{(11)}, l^{(-11)}$ are constants, $l^{(11)}$ is the loss
output on identical input pair: $l^{(11)}=l(1,1)$, and $l^{(-11)}$ is
the loss output on distinct input pair: $l^{(-11)}=l(-1,1)$.
\end{proposition}
\begin{proof}
    This proposition can be easily proved by exhaustively checking all
    possible inputs of the loss function. Notice that there are only
    two possible output values of the loss function.
For the input $(z_1=1,z_2=1)$:
 \begin{align}
	g(1,1) & = \frac{1}{2} \biggr[ 1 \times 1 \times  (l^{(11)} - l^{(-11)}) + l^{(11)} + l^{(-11)} \biggr] \notag \\
	& = l(1,1), \notag
  \end{align}
For the input $(z_1=-1,z_2=1)$:
 \begin{align}
	g(-1,1) & = \frac{1}{2} \biggr[ -1 \times 1 \times (l^{(11)} - l^{(-11)}) + l^{(11)} + l^{(-11)} \biggr] \notag \\
	& = l(-1,1), \notag
  \end{align}
The input $(z_1=-1,z_2=-1)$ is the same as $(z_1=1,z_2=1)$ and the input $(z_1=1,z_2=-1)$ is the same as $(z_1=-1,z_2=1)$.
In conclusion, the function $l(\cdot,\cdot)$ and $g(\cdot,\cdot)$ have the same output for any possible inputs.
\end{proof}

Any hash loss function $l(\cdot, \cdot)$ which is defined on the Hamming affinity or
Hamming distance of data pairs is able to meet the
requirement that: $l(1,1)=l(-1,-1), l(1,-1)=l(-1,1)$.
Applying this
proposition, the optimization of \eqref{tsh-eq:opt_step1-onebit} can be equivalently
reformulated as:
\begin{subequations}
 \begin{align}
 	\label{tsh-eq:bqp-org}
	& \min_{\z_{(r)} \in \{-1, 1\}^n} \sum_{i=1}^n\sum_{j=1}^n a_{i,j} z_{r,i}z_{r,j}, \\
	& \text{where, } a_{i,j}=\relationindicator (l_{r,i,j}^{(11)} - l_{r,i,j}^{(-11)}), \label{eq:matrix_aij} \\
	& \quad l_{r,i,j}^{(11)}= l_{r}( 1, 1; y_{ij} ), \quad  l_{r,i,j}^{(-11)}= l_{r}(-1, 1; y_{ij} ).
	\label{eq:lrij}
 \end{align}
\end{subequations}
Here $a_{i,j}$ is constant.
The above optimization is an unconstrained binary quadratic problem.
It can be written in a matrix form:
\begin{subequations}
 \begin{align}
 	\label{tsh-eq:bqp}
	& \min_{\z_{(r)}} \z_{(r)}^\T \A \z_{(r)}, \\
    \st & \z_{(r)} \in \{-1, 1\}^n.
 \end{align}
\end{subequations}
Here the $ (i,j) $-th element of matrix $\A$ is defined by $a_{i,j}$ in \eqref{eq:matrix_aij}.
We have shown that the original optimization in
\eqref{tsh-eq:opt_step1-onebit} for one bit can be equivalently reformulated as a binary quadratic
problem (BQP) in \eqref{tsh-eq:bqp-org}. We discuss algorithms for solving this BQP in the next section.

Here we describe
a selection of such loss functions,
most of which arise from recent hashing methods.
These loss functions are defined on Hamming distance/affinity,
thus they are applicable to Proposition \ref{tsh-pro:p1}.
Recall that $m$ is the number of bits, $\dhamm(\cdot,\cdot)$ is the Hamming distance
and $\delta(\cdot) \in \{0, 1\}$ is an indicator function.

{\bf FastH-KSH}
The KSH loss function is based on Hamming affinity.
MDSH also uses a similar form of loss function (weighted Hamming affinity instead).
    \begin{align}
    \label{eq:loss-ksh}
	L_{\mathrm{KSH}}(\z_i, \z_j)= (my_{ij} - \z_i^\T\z_j)^2.
 \end{align}

{\bf FastH-Hinge}
The Hinge loss function is based on Hamming distance:
\begin{align}
\label{eq:loss-hinge}
	 L_{\mathrm{Hinge}}(\z_i, \z_j) =
	   \begin{cases}
	     [ 0 - \dhamm(\z_i,\z_j) ]^2 & \text{if } y_{ij}>0, \\
	     [ \max(0.5m - \dhamm(\z_i,\z_j), 0 ) ]^2  & \text{if } y_{ij}<0.
		\end{cases}
\end{align}

{\bf FastH-BRE}
The BRE loss function is based on Hamming distance:
 \begin{align}
 	\label{eq:loss-bre}
	L_{\mathrm{BRE}}(\z_i, \z_j)= [ m\delta(y_{ij}<0) - \dhamm(\z_i,\z_j) ]^2.
 \end{align}

{\bf FastH-ExpH}
Here ExpH is an exponential loss function using the Hamming distance:
\begin{align}
	\label{eq:loss-exph}
	L_{\mathrm{ExpH}}(\z_i, \z_j)= \exp [{y_{ij}\dhamm(\z_i,\z_j)/m +
    \delta(y_{ij}<0)} ].
 \end{align}

These loss functions are evaluated in the experiment section later.
It is worth noting that the Hinge loss \eqref{eq:loss-hinge} encourages the
Hamming distance of dissimilar pairs to be \emph{at least} more than \emph{half} of the bit length.
This is plausible because the Hamming distance of dissimilar pairs is only required to be large enough,
but not necessarily be the maximum value.
In contrast, other regression-like loss functions (e.g., KSH, BRE)
push the distance of dissimilar pairs to the maximum value (the bit length),
which may introduce unnecessary penalties.
As empirically verified in our experiments, the Hinge loss usually performs better.

One motivation of KSH in \cite{KSH} for using the Hamming affinity based loss function rather than the Hamming distance is that they can apply efficient optimization algorithms. However, here we show that both Hamming affinity and Hamming distance based loss functions can be easily solved in our general two-step framework using identical optimization techniques.

To apply the result of Proposition \ref{tsh-pro:p1} in \eqref{tsh-eq:bqp-org},
we take the KSH loss function as an example.
Recall that $l_r ( z_{r,i} , z_{r,j} ; y_{ij} )$ in \eqref{tsh-eq:opt_step1-onebit}
is the loss function output for the $r$-th bit and data pair $(i, j)$.
Using the KSH loss in \eqref{eq:loss-ksh}, we have:
\begin{align}
 l_r ( z_{r,i} , z_{r,j} ; y_{ij} ) =
 (ry_{ij} - \z_{i}^{(r-1) \T } \z_{j}^{(r-1)} - z_{r,i} z_{r,j} )^2
 \end{align}
Recall that $l^{(11)}$ is the loss
output on identical input pairs, and $l^{(-11)}$ is
the loss output on distinct input pairs.
With the above equation, we can write $l_{r,i,j}^{(11)}$ and $l_{r,i,j}^{(-11)}$ in \eqref{eq:lrij} for the KSH loss as:
\begin{subequations}
\begin{align}
 l_{r,i,j}^{(11)} & = l_{r} ( 1, 1; y_{ij} ) \notag \\
	& = (ry_{ij} - \z_{i}^{(r-1) \T } \z_{j}^{(r-1)} - 1 )^2;\\
 l_{r,i,j}^{(-11)} & = l_{r} ( -1, 1; y_{ij} )  \notag \\
 	& =  (ry_{ij} - \z_{i}^{(r-1) \T } \z_{j}^{(r-1)} + 1 )^2.
 \end{align}
 \end{subequations}
Finally, the matrix element $a_{i,j}$ in the BQP problem \eqref{tsh-eq:bqp-org} is written as:
\begin{subequations}
\label{eq:aij_ksh}
\begin{align}
 a_{i,j}^{\mathrm{KSH}}= & \relationindicator (l_{r,i,j}^{(11)} - l_{r,i,j}^{(-11)}) \\
 = & \relationindicator [ \, (ry_{ij} - \z_{i}^{(r-1) \T } \z_{j}^{(r-1)} - 1 )^2  \notag \\
 & \quad \quad - (ry_{ij} - \z_{i}^{(r-1) \T } \z_{j}^{(r-1)} + 1 )^2 \, ] \\
 = & -4 \relationindicator (ry_{ij} - \z_{i}^{(r-1) \T } \z_{j}^{(r-1)}).
 \end{align}
 \end{subequations}
 By substituting into \eqref{tsh-eq:bqp-org} and removing constant multipliers,
 we obtain the binary code inference problem for the $r$-th bit using the KSH loss function:
\begin{subequations}
\begin{align}
 	\label{tsh-eq:bqp-ksh}
	 & \min_{\z_r \in \{-1, 1\}^{n}} \ssum_{i=1}^n\sum_{j=1}^n
      a_{i,j} z_{r, i}z_{r, j}, \\
      \text{where,} \; & a_{i,j}  = -\relationindicator ( r y_{ij} -
      \ssum_{p=1}^{r-1} z^\ast_{p, i} z^\ast_{p, j} ). \label{tsh-eq:bqp-ksh2}
\end{align}
\end{subequations}
Here $z^\ast$ denotes the binary code in previous bits.

\subsection{Spectral method for binary inference}
To solve the BQP problem in \eqref{tsh-eq:bqp} for obtaining binary codes,
we first describe a simple spectral relaxation based method,
then present an efficient graph cut based method for large-scale inference.
Spectral relaxation drops the binary constraints. The optimization becomes:
 \begin{align}
 	\label{tsh-loss7}
	\min_{\z_{(r)}} \z_{(r)}^\T \A \z_{(r)}, \;\;\; \notag \\
    \st \;\; \|\z_{(r)}\|_2^2=n.
 \end{align}
The solution (denoted $\z_{(r)}^0$) of the above optimization is
simply the eigenvector that corresponds to the minimum eigenvalue of
the matrix $\A$.
To achieve a better solution,
we can solve the following relaxed problem of
\eqref{tsh-eq:bqp}:
 \begin{align}
 	\label{tsh-loss8}
	\min_{\z_{(r)}} \z_{(r)}^\T \A \z_{(r)}, \;\;\; \notag \\
    \st \;\; \z_{(r)} =[-1, 1]^n.
 \end{align}
We use the solution $\z_{(r)}^0$ of
spectral relaxation in \eqref{tsh-loss7} as an initialization and solve the above problem
using the efficient LBFGS-B solver \cite{lbfgs}.
The solution is then thresholded at $0$ to output the final binary codes.

\subsection{Block GraphCut for binary code inference}

\label{tsh-sec:step1-gc}

We have shown that the simple spectral method can be used
to solve the binary code inference problem in \eqref{tsh-eq:bqp}.
However solving eigenvalue problems does not scale up to large training sets,
and the loose relaxation leads to inferior results.
Here we propose sub-modular formulations
and an efficient graph cut based block search method for solving large-scale inference problems.
This block search method is much more efficient than the spectral method and able to achieve better solutions.

Specifically, we first group data points into a number of blocks,
then iteratively optimize for these blocks until converge.
At each iteration,
we randomly pick one block,
then optimize for (update) the corresponding binary variables of this block,
conditioning on
the remaining variables.
In other words,
when optimizing for one block, only those binary variables that correspond to the data points of the target block will be updated;
and for the variables that are not involved in the target block,
their values remain unchanged.
Clearly each block update would strictly decrease the objective.

Formally, let $\block$ denote a block of data points.
We want to optimize for the corresponding binary variables of the block $\block$.
We denote by $\hat z_r$ a binary code in the $r$-bit that is not involved in the target block.
First we rewrite the objective in \eqref{tsh-eq:bqp-org}
to separate the variables of the target block from other variables.
The objective in \eqref{tsh-eq:bqp-org} can be rewritten as:
\begin{subequations}
\begin{align}
	& \sum_{i=1}^n\sum_{j=1}^n a_{i,j} z_{r, i}z_{r, j} \\
      = & \sum_{i \in \block} \sum_{j \in \block} a_{i,j} z_{r, i}z_{r, j}
      	 + \sum_{i \in \block} \sum_{j \notin \block} a_{i,j} z_{r, i} \hat z_{r, j} \notag \\
      	 & \;\;\; + \sum_{i \notin \block} \sum_{j \in \block} a_{i,j}  z_{r, i} \hat z_{r, j}
      	 + \sum_{i \notin \block} \sum_{j \notin \block} a_{i,j} \hat  z_{r, i} \hat  z_{r, j} \\
      	 = & \sum_{i \in \block} \sum_{j \in \block} a_{i,j} z_{r, i} z_{r, j}
      	 + 2 \sum_{i \in \block} \sum_{j \notin \block} a_{i,j} z_{r, i} \hat z_{r, j} \notag \\
      	 & \;\;\; +  \sum_{i \notin \block} \sum_{j \notin \block} a_{i,j} \hat  z_{r, i} \hat  z_{r, j}.
\end{align}
\end{subequations}
When optimizing for one block, those variables which are not involved in the target block are treated as constants;
hence $\hat z_r$ is treated as a constant.
By removing the constant part, the optimization for one block is:
\begin{align}
 	\label{fh-eq:opt_step1_block_tmp}
	\min_{\z_{r, \block} \in \{-1, 1\}^{|\block|}}
	\sum_{i \in \block} \sum_{j \in \block} a_{i,j} z_{r, i} z_{r, j}
      	 + 2 \sum_{i \in \block} \sum_{j \notin \block} a_{i,j} z_{r, i} \hat z_{r, j}.
\end{align}
We aim to optimize $\z_{r, \block}$ which is a vector of variables that are involved in the target block $\block$.
 Substituting the constant $a_{i,j}$ by its definition in \eqref{eq:matrix_aij},
 the above optimization is written as:
\begin{subequations}
\label{fh-eq:opt_step1_block_all}
\begin{align}
 	\label{fh-eq:opt_step1_block}
	 \min_{\z_{r, \block} \in \{-1, 1\}^{|\block|}}  & \sum_{i \in \block} u_i z_{r,i}
	+  \sum_{i \in \block}\sum_{j \in \block}
      v_{ij} z_{r,i} z_{r,j}, \\
       \text{where,}  \;   v_{ij} & = \relationindicator (l_{r,i,j}^{(11)} - l_{r,i,j}^{(-11)})
       \label{fh-eq:opt_step1_block2} \\
         u_i & =  2 \sum_{j \notin \block} \hat  z_{r,j} \relationindicator (l_{r,i,j}^{(11)} - l_{r,i,j}^{(-11)}).
         \label{fh-eq:opt_step1_block3}
\end{align}
\end{subequations}
Here $u_i, v_{ij}$ are constants.
The key to constructing a block is to ensure \eqref{fh-eq:opt_step1_block} for
such a block is sub-modular, thus we are able to apply the efficient graph cut method.
We refer to this as Block GraphCut (Block-GC), shown in Algorithm \ref{fh-alg:step1}.
Specifically in our hashing problem, by leveraging similarity information, we can easily
construct blocks to meet the sub-modularity requirement.
Here we assume that the loss function satisfies the following conditions:
\begin{align}
	\label{eq: loss_assume}
	\forall y_{ij} \geq 0 & \text{ and } \forall r : \notag \\
	& a_{i,j}=\relationindicator (l_{r,i,j}^{(11)} - l_{r,i,j}^{(-11)}) \leq 0,
\end{align}
which intuitively means that, for two similar data points, the loss of assigning identical binary values for one bit is smaller than assigning distinct binary values.
As loss functions always encourage two similar data points to have similar binary codes,
this condition can be naturally satisfied.
All of the loss functions (e.g., KSH, BRE, Hinge) that we described before meet this requirement.
As an example, the definition of $a_{i,j}^{\rm KSH}$ in \eqref{eq:aij_ksh} for the KSH loss satisfies the above conditions.
The following proposition shows how to construct such a block:
\begin{proposition}
\label{fh-pro:p1}
	$\forall i,j \in \block$, if $y_{ij} \geq 0 $,
	the optimization in \eqref{fh-eq:opt_step1_block} is a sub-modular problem.
In other words, for any data point in the block, if it is {\emph not} dissimilar with any other data points in the block,
   then \eqref{fh-eq:opt_step1_block} is sub-modular.
\end{proposition}

\begin{proof}
If  $y_{ij} \geq 0$, according to the conditions in \eqref{eq: loss_assume},
we have:
\begin{align}
	v_{ij}  = \relationindicator (l_{r,i,j}^{(11)} - l_{r,i,j}^{(-11)}) \leq 0.
\end{align}
With the following definition:
\begin{align}
	\theta_{i,j}(z_{r,i},z_{r,j})=v_{ij} z_{r,i} z_{r,j},
\end{align}
the following holds:
\begin{align}
	& \theta_{i,j}(-1,1) = \theta_{i,j}(1,-1)  =  -v_{ij} \geq 0; \\
	& \theta_{i,j}(1,1) = \theta_{i,j}(-1,-1)  =  v_{ij} \leq 0.
\end{align}
Hence we have the following relations: $ \forall i,j \in \block $:
\begin{align}
	\theta_{i,j}(1,1) + \theta_{i,j}(-1,-1) \leq 0 \leq \theta_{i,j}(1,-1) + \theta_{i,j}(-1,1),
\end{align}
which prove the sub-modularity of \eqref{fh-eq:opt_step1_block} \cite{rother2007optimizing}.
\end{proof}
Blocks can be constructed in many ways as long as they satisfy the condition in Proposition \ref{fh-pro:p1}.
A simple greedy method is shown in Algorithm \ref{fh-alg:block}.
Note that one block can overlap with another and the union of all blocks needs to cover all $n$ variables.
If one block only consist of one variable, Block-GC becomes the ICM method (\cite{besag1986statistical, UGM}) which optimizes for one variable at a time.

\section{Step 2: hash function learning}

\label{tsh-sec:step2}

The second step is to solve a binary classification problem for learning one hash function.
The binary codes obtained in the first step are used as the classification labels.
Any binary classifiers (e.g., decision trees, neural networks) and
any advanced large-scale training techniques can be directly applied to hash function learning at this step.
For the $r$-th bit, the classification problem is:
\begin{align}
	\min_{h_r(\cdot)} \;\; & \ssum_{i=1}^n \delta( z_{r,i} = h_r(\x_i) ).
    \label{tsh-eq:opt_step2_onebit}
\end{align}
Usually the zero-one loss in the above problem is replaced by some convex surrogate loss.
For example, when training a perceptron hash function:
\begin{align}
	\label{eq:hashfun_lsvm}
	h(\x) = \sign (\w^\T\x + b),
 \end{align}
we can train a linear SVM classifier by solving:
\begin{align}
	\min_{\w, b} \;\; & \frac{1}{2}\|\w\|^2 + \ssum_{i=1}^n \max[1-z_{r, i} (\w^\T \x +b), 0].
    \label{eq:lsvm}
\end{align}

Any binary classifier can be applied here.
We could also train an kernel SVM to learn a kernel hash function:
 \begin{align}
 	\label{eq:hashfun_kernel}
	h(\x) = \sign\left[  \sum_{q=1}^Q w_q\kappa(\x_q', \x) + b \right],
 \end{align}
in which $\cX'=\{\x'_1,\dots,\x'_Q\}$ are $Q$ support vectors.
Sophisticated kernel learning methods can be applied here, for example, LIBSVM or
the stochastic kernel SVM training method with a support vector budget in \cite{wang2012breaking}.

After learning the hash function for one bit,
the binary code is updated by applying the learned hash function.
Hence, the learned hash function is able to influence the learning of the next bit.

\subsection{Boosted trees as hash functions}

\label{tsh-sec:step2-tree}

Decision trees could be a good choice for hash functions with nonlinear mapping.
Compared to kernel method, decision trees only involve simple comparison
operations for evaluation; thus they are much more efficient for testing, especially on high-dimensional data.
We define one hash function as a linear combination of trees:
\begin{align}
 	\label{eq:hashfun_btree}
	h(\x)=\sign \biggr[ \ssum_{q=1}^Q  w_q \tree_q(\x) \biggr].
\end{align}
Here $Q$ is the number of decision trees;
$\tree(\cdot) \in \{-1, 1\}$ denotes a tree function with binary output.
We train a boosting classifier to learn
the weighting coefficients and trees for one hash function.
The classification problem for the $r$-th hash function is written as:
\begin{align}
 	\label{eq:adaboost}
	\min_{\w \geq 0} \ssum_{i=1}^n \exp \biggr[-z_{r,i} \ssum_{q=1}^Q  w_q \tree_q(\x_i) \biggr].
\end{align}
We apply Adaboost to solve the problem.
At each boosting iteration, a decision tree as well as its weighting coefficient is learned.
Every node of a binary decision tree is a decision stump.
Training a stump is to find a feature dimension and threshold that minimizes the weighted classification error.
From this point of view, we are performing feature selection and hash function learning at the same time.
We can easily make use of efficient decision tree learning techniques available in the literature.
Here we summarize some techniques that are included in our implementation:

1) We use the efficient stump implementation proposed in the recent work of \cite{appelquickly}, which is around 10 times faster than conventional implementation.

2) Feature quantization can significantly speed up tree training without noticeable performance loss in practice, and also largely reduce the memory consuming. As in \cite{appelquickly}, we linearly quantize feature values into $256$ bins.

3) We apply the weight-trimming technique described in \cite{friedman2000additive,appelquickly}.
At each boosting iteration, the smallest $10\%$ weightings are trimmed (set to $0$).

4) We apply the LazyBoost technique to speed up the tree learning process.
For one node splitting in tree training, only a random subset of feature dimensions are evaluated for splitting.

\begin{figure}[t]
    \centering

        \figcenter{\includegraphics[height=.3in]{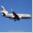}}
        \colseperator
        \figcenter{\includegraphics[width=.85\linewidth]{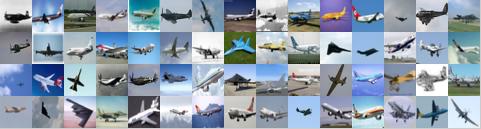}}

        \rowseperator
        \figcenter{\includegraphics[height=.3in]{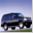}}
        \colseperator
        \figcenter{\includegraphics[width=.85\linewidth]{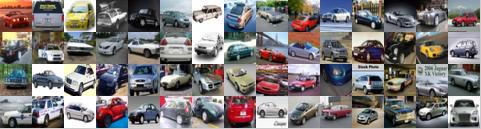}}

        \rowseperator
        \figcenter{\includegraphics[height=.3in]{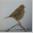}}
        \colseperator
        \figcenter{\includegraphics[width=.85\linewidth]{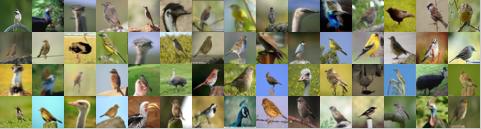}}

        \rowseperator
        \figcenter{\includegraphics[height=.3in]{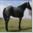}}
        \colseperator
        \figcenter{\includegraphics[width=.85\linewidth]{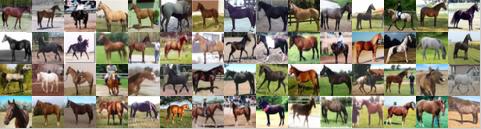}}

          \rowseperator
        \figcenter{\includegraphics[height=.3in]{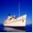}}
        \colseperator
        \figcenter{\includegraphics[width=.85\linewidth]{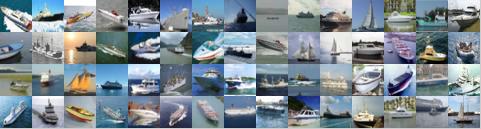}}

    \caption{Some retrieval examples of our method \fasth on CIFAR10. The first column shows query images, and the rest are retrieved images in the database.
    }
    \label{fh-fig:example_cifar}
\end{figure}

\begin{figure*}[t]

    \centering

  \includegraphics[width=.245\linewidth]{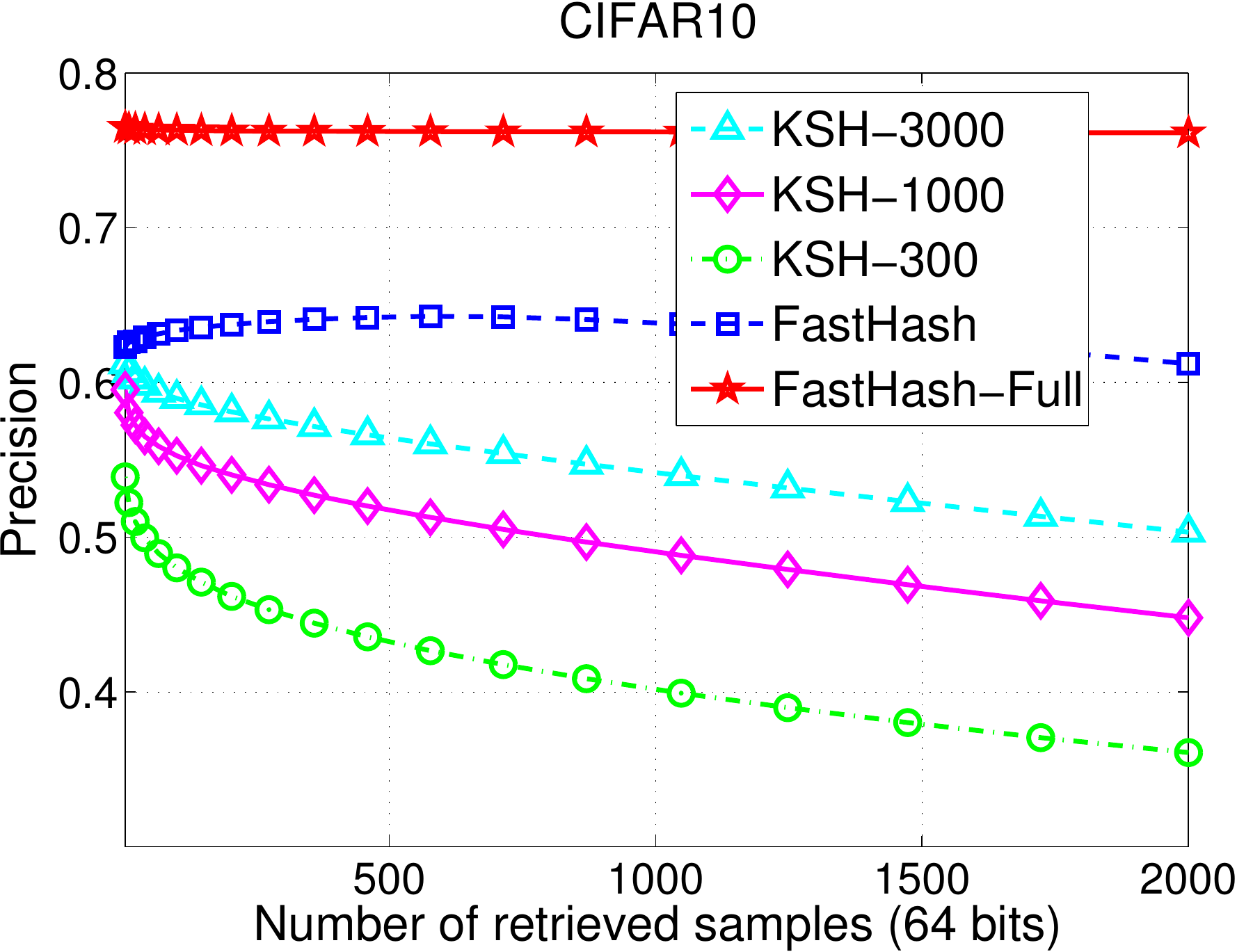}
  \includegraphics[width=.245\linewidth]{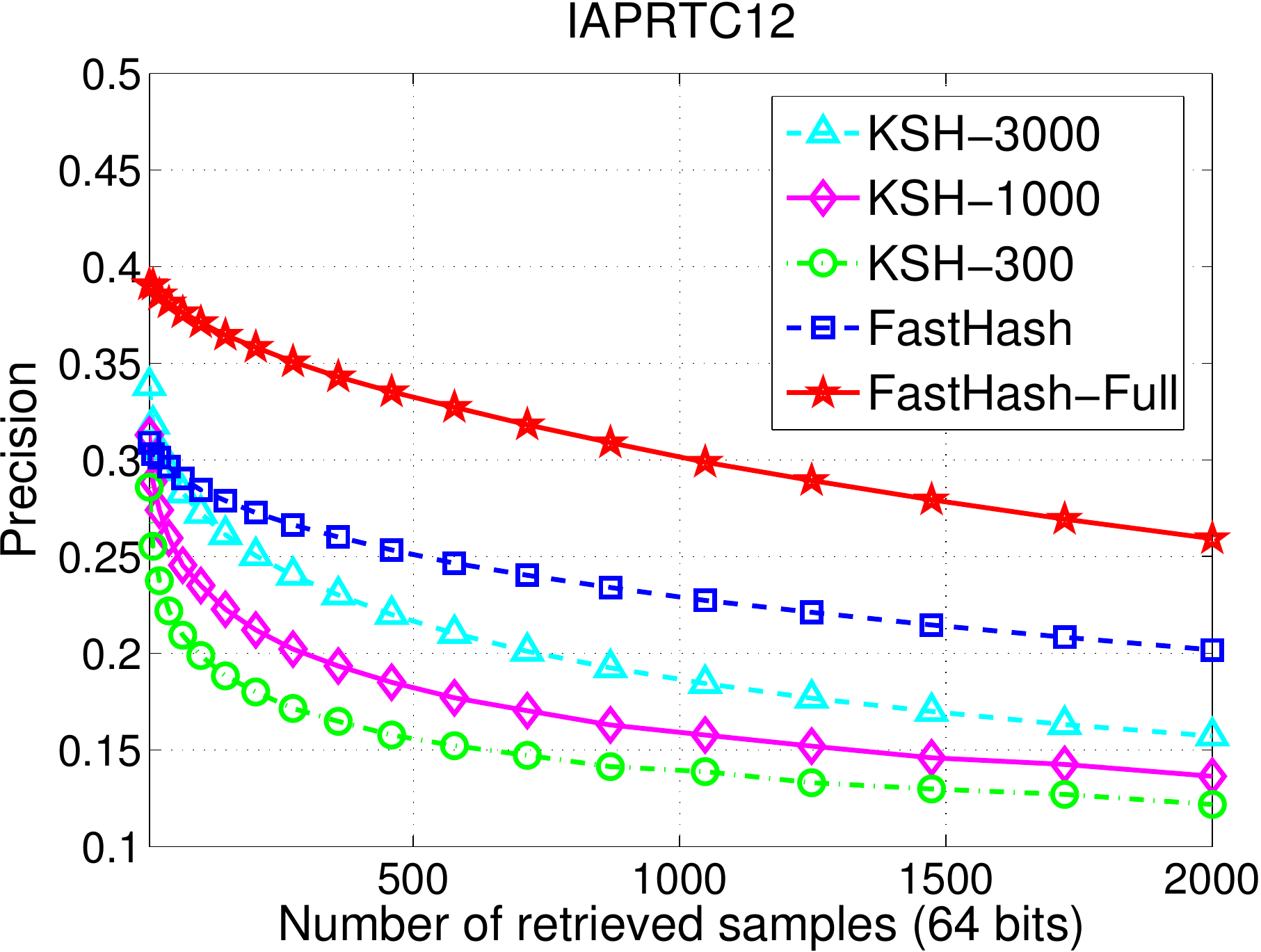}
  \includegraphics[width=.245\linewidth]{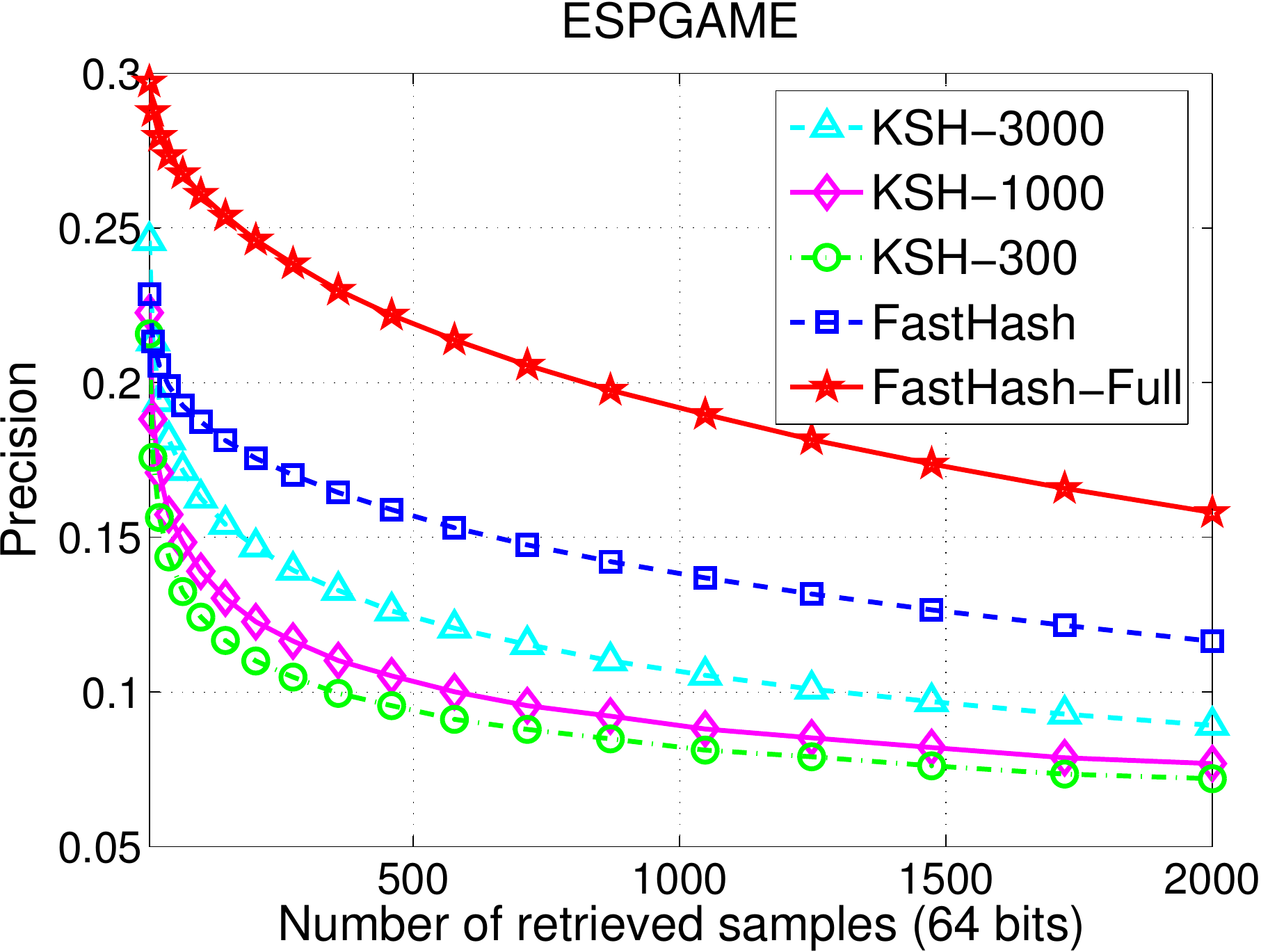}
  \includegraphics[width=.245\linewidth]{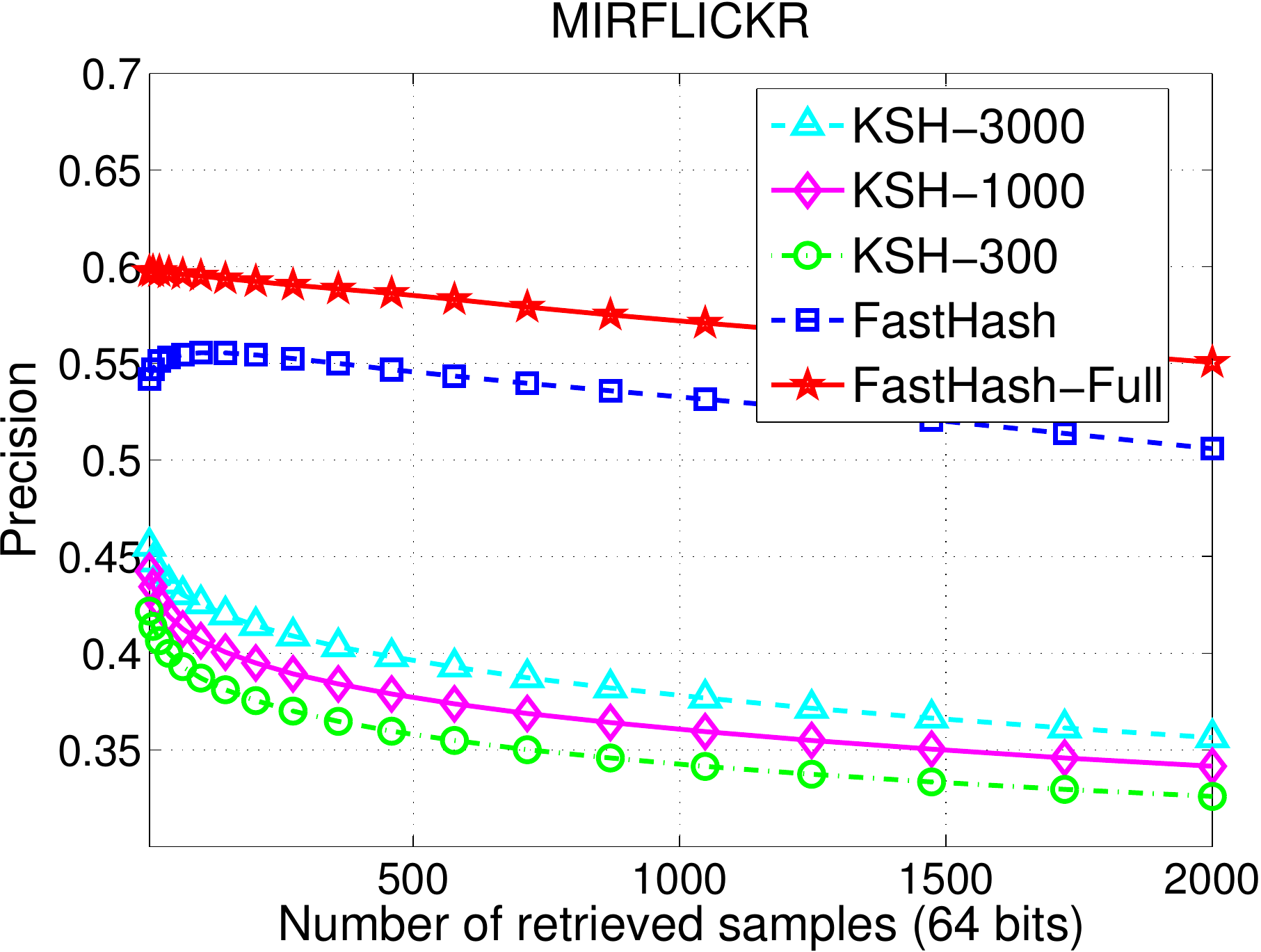}

  \caption{Comparison of KSH and our \fasth on all datasets.
    The number after ``KSH'' is the number of support vectors.
Both of our \fasth and FastHash-Full significantly outperform KSH.}
    \label{fh-fig:ksh}
\end{figure*}

\begin{table}[t]
\caption{Comparison of KSH and our \fasth. KSH results
are presented with different numbers of support vectors.
Both of our \fasth and FastHash-Full significantly outperform KSH in terms of training time, binary encoding time (test time) and retrieval precision.}
\centering
 \resizebox{.98\linewidth}{!}
  {
  \begin{tabular}{ l | l c | l c c}
  \hline\hline
Method  &\#Train  &\#Support Vector &Train time &Test time  &Precision\\
\hline
\multicolumn{6}{  c }{CIFAR10 (features:11200)} \\ \hdashline
KSH &5000 &300  &1082 &22 &0.480\\
KSH &5000 &1000 &3481 &57 &0.553\\
KSH &5000 &3000 &52747  &145  &0.590\\
\best FastH &5000 &N/A  &331  &21 &\best 0.634\\
\best FastH-Full  &50000  &N/A  &1794 &21 &\best 0.763\\
\hline
\multicolumn{6}{  c }{IAPRTC12 (features:11200)} \\ \hdashline
KSH &5000 &300  &1129 &7  &0.199\\
KSH &5000 &1000 &3447 &21 &0.235\\
KSH &5000 &3000 &51927  &51 &0.273\\
\best FastH &5000 &N/A  &331  &9  &\best 0.285\\
\best FastH-Full  &17665  &N/A  &620  &9  &\best 0.371\\
\hline
\multicolumn{6}{  c }{ESPGAME (features:11200)} \\ \hdashline
KSH &5000 &300  &1120 &8  &0.124\\
KSH &5000 &1000 &3358 &22 &0.139\\
KSH &5000 &3000 &52115  &46 &0.163\\
\best FastH &5000 &N/A  & 309 &9  &\best 0.188\\
\best FastH-Full  &18689  &N/A  &663  &9  &\best 0.261\\
\hline
\multicolumn{6}{  c }{MIRFLICKR (features:11200)} \\ \hdashline
KSH &5000 &300  &1036 &5  &0.387\\
KSH &5000 &1000 &3337 &13 &0.407\\
KSH &5000 &3000 &52031  &42 &0.434\\
\best FastH &5000 &N/A  &278  &7  &\best0.555\\
\best FastH-Full  &12500  &N/A  &509  &7  &\best 0.595\\

  \hline \hline
  \end{tabular}
  }
\label{fh-tab:ksh}
\end{table}

\section{Experiments}

To evaluate
the proposed method,
here we present the results of comprehensive experiments on several large image datasets.
The evaluation measures include training time, binary encoding time and retrieval accuracy.
We compare to a number of recent supervised and unsupervised hashing methods.
To explore our method with different settings,
we perform comparisons of using different binary code inference algorithms,
various kinds of loss functions and hash functions.

The similarity preserving performance is evaluated
in small binary codes based image retrieval \cite{torralba2008small,wang2010semi}.
Given a query image, the retrieved images in the database are returned by hamming distance ranking
based on their binary codes.
The retrieval quality is measured in 3 different aspects: the precision of the top-K ($K$=$100$)
retrieved examples (denoted as Precision), mean average precision (MAP) and the area under the Precision-Recall curve (Prec-Recall).
The training time and testing time (binary encoding time) are recorded in seconds.

Results are reported on 6 large image datasets which cover a wide variety of images.
The dataset CIFAR10
\footnote{\url{http://www.cs.toronto.edu/~kriz/cifar.html}}
contains $60,000$ images in small resolution.
The multi-label datasets IAPRTC12 and ESPGAME \cite{guillaumin2009tagprop}
contain around $20,000$ images, and MIRFLICKR \cite{huiskes2008mir} is a collection of $25000$ images.
SUN397 \cite{xiao2010sun}
contains more than $100,000$ scene images
form $397$ categories.
The large dataset ILSVRC2012 contains roughly $1.2$ million images of $1000$ categories from ImageNet \cite{imagenet}.

For the multi-class datasets: CIFAR10, SUN397 and ILSVRC2012, the ground truth pairwise similarity is defined as multi-class label agreement. For multi-label datasets: IAPRTC12, ESPGAME and MIRFLICKR, of which the keyword (tags) annotation is provided in \cite{guillaumin2009tagprop}, two images are considered as semantically similar if they are annotated with at least $2$ identical keywords (or tags).
In the training stage of supervised methods,
a maximum number of 100 similar and dissimilar neighbors are defined for each example;
hence the pairwise similarity label matrix is sparse.

Following the conventional protocol in \cite{KSH,kulis2009learning,wang2010semi} for hashing method evaluation,
a large portion of the dataset is
allocated as an image database for training and retrieval,
and the rest is put aside as test queries.
Specifically, for CIFAR10, IAPRTC12, ESPGAME and MIRFLICKER, the training data in the the provided split are used as image database and the test data are used as test queries.
The splits for SUN397 and ILSVRC2012 are described in their corresponding sections.

\begin{table*}[t]
\caption{Results using two types of features: low-dimensional GIST features and the
  high-dimensional codebook features. Our \fasth and FastHash-Full significantly outperform
  the comparators on both feature types.
  In terms of training time, our \fasth is also much faster than others on the high-dimensional codebook features.}
\centering
\resizebox{.95\linewidth}{!}
  {
  \begin{tabular}{ c l || l l c c c | l l c c c}
  \hline \hline
    & & \multicolumn{5}{ c | }{GIST feature (320 / 512 dimensions)}
& \multicolumn{5}{ c }{Codebook feature (11200 dimensions)}
  \\ \hline
Method &\#Train &Train time &Test time  &Precision  &MAP  &Prec-Recall  &Train time (s) &Test time (s)
 &Precision  &MAP  &Prec-Recall
\\
\hline
\multicolumn{12}{ c }{CIFAR10} \\  \hdashline
KSH &5000 &52173  &8  &0.453  &0.350  &0.164  &52747  &145  &0.590  &0.464  &0.261  \\
BREs  &5000 &481  &1  &0.262  &0.198  &0.082  &18343  &8  &0.292  &0.216  &0.089  \\
SPLH  &5000 &102  &1  &0.368  &0.291  &0.138  &9858 &4  &0.496  &0.396  &0.219  \\
STHs  &5000 &380  &1  &0.197  &0.151  &0.051  &6878 &4  &0.246  &0.175  &0.058  \\
\best FastH &5000 &304  &21 &\best 0.517  &\best 0.462  &\best 0.243  &331  &21 &\best 0.634  &\best 0.575  &\best 0.358  \\
\best FastH-Full  &50000  &1681 &21 &\best 0.649  &\best 0.653  &\best 0.450  &1794 &21 &\best 0.763  &\best 0.775  &\best 0.605  \\
\hline
\multicolumn{12}{ c }{IAPRTC12} \\  \hdashline
KSH &5000 &51864  &5  &0.182  &0.126  &0.083  &51927  &51 &0.273  &0.169  &0.123  \\
BREs  &5000 &6052 &1  &0.138  &0.109  &0.074  &6779 &3  &0.163  &0.124  &0.097  \\
SPLH  &5000 &154  &1  &0.160  &0.124  &0.084  &10261  &2  &0.220  &0.157  &0.119  \\
STHs  &5000 &628  &1  &0.099  &0.092  &0.062  &10108  &2  &0.160  &0.114  &0.076  \\
\best FastH &5000 &286  &9  &\best 0.232  &\best 0.168  &\best 0.117  &331  &9  &\best 0.285  &\best 0.202  &\best 0.146  \\
\best FastH-Full  &17665  &590  &9  &\best 0.316  &\best 0.240  &\best 0.178  &620  &9  &\best 0.371  &\best 0.276  &\best 0.210  \\
\hline
\multicolumn{12}{ c }{ESPGAME} \\  \hdashline
KSH &5000 &52061  &5  &0.118  &0.077  &0.054  &52115  &46 &0.163  &0.100  &0.072  \\
BREs  &5000 &714  &1  &0.095  &0.070  &0.050  &16628  &3  &0.111  &0.076  &0.059  \\
SPLH  &5000 &185  &1  &0.116  &0.083  &0.062  &11740  &2  &0.148  &0.104  &0.074  \\
STHs  &5000 &616  &1  &0.061  &0.047  &0.033  &11045  &2  &0.087  &0.064  &0.042  \\
\best FastH &5000 &289  &9  &\best 0.157  &\best 0.106  &\best 0.070  &309  &9  &\best 0.188  &\best 0.125  &\best 0.081  \\
\best FastH-Full  &18689  &448  &9  &\best 0.228  &\best 0.169  &\best 0.109  &663  &9  &\best 0.261  &\best 0.189  &\best 0.126  \\
\hline
\multicolumn{12}{ c }{MIRFLICKR} \\  \hdashline
KSH &5000 &51983  &3  &0.379  &0.321  &0.234  &52031  &42 &0.434  &0.350  &0.254  \\
BREs  &5000 &1161 &1  &0.347  &0.310  &0.224  &13671  &2  &0.399  &0.345  &0.250  \\
SPLH  &5000 &166  &1  &0.379  &0.337  &0.241  &9824 &2  &0.444  &0.391  &0.277  \\
STHs  &5000 &613  &1  &0.268  &0.261  &0.172  &10254  &2  &0.281  &0.272  &0.174  \\
\best FastH &5000 &307  &7  &\best 0.477  &\best 0.429  &\best 0.299  &338  &7  &\best 0.555  &\best 0.487  & \best 0.344 \\
\best FastH-Full  &12500  &451  &7  &\best 0.525  &\best 0.507  &\best 0.345  & 509 &7  &\best 0.595  &\best 0.558  &\best 0.420  \\
  \hline \hline
  \end{tabular}
  }
\label{fh-tab:features}
\end{table*}

We extract codebook-based features following the conventional pipeline from \cite{coates2011importance, kiros12}:
we employ K-SVD for codebook (dictionary) learning with a codebook size of $800$, soft-thresholding for patch encoding and spatial pooling of $3$ levels, which results $11200$-dimensional features. For further evaluation, we
increase the codebook size to $1600$ to generate $22400$-dimensional features.
We also extract the low-dimensional GIST \cite{gist} features ($512$ or $320$ dimensions) for evaluation.
For the dataset ILSVRC2012, we extract the convolution neural network features with $4096$ dimensions
(\cite{krizhevsky2012imagenet,Donahue14}), which is described in detail in the corresponding section.

If not specified,
we use the following in our method:
the KSH loss function (described in \eqref{eq:loss-ksh}), the proposed Block GraphCut algorithm in Step-1, and the decision tree hash function in Step-2.
The tree depth is set to $4$, and the number of boosting iterations is $200$.
Different settings of hash functions or loss functions will be evaluated in the later sections.
For comparison methods, we follow the original papers for parameter setting.
If not specified, $64$-bit binary codes are generated for evaluation.

\begin{figure*}[t]
    \centering

   \includegraphics[width=.245\linewidth]{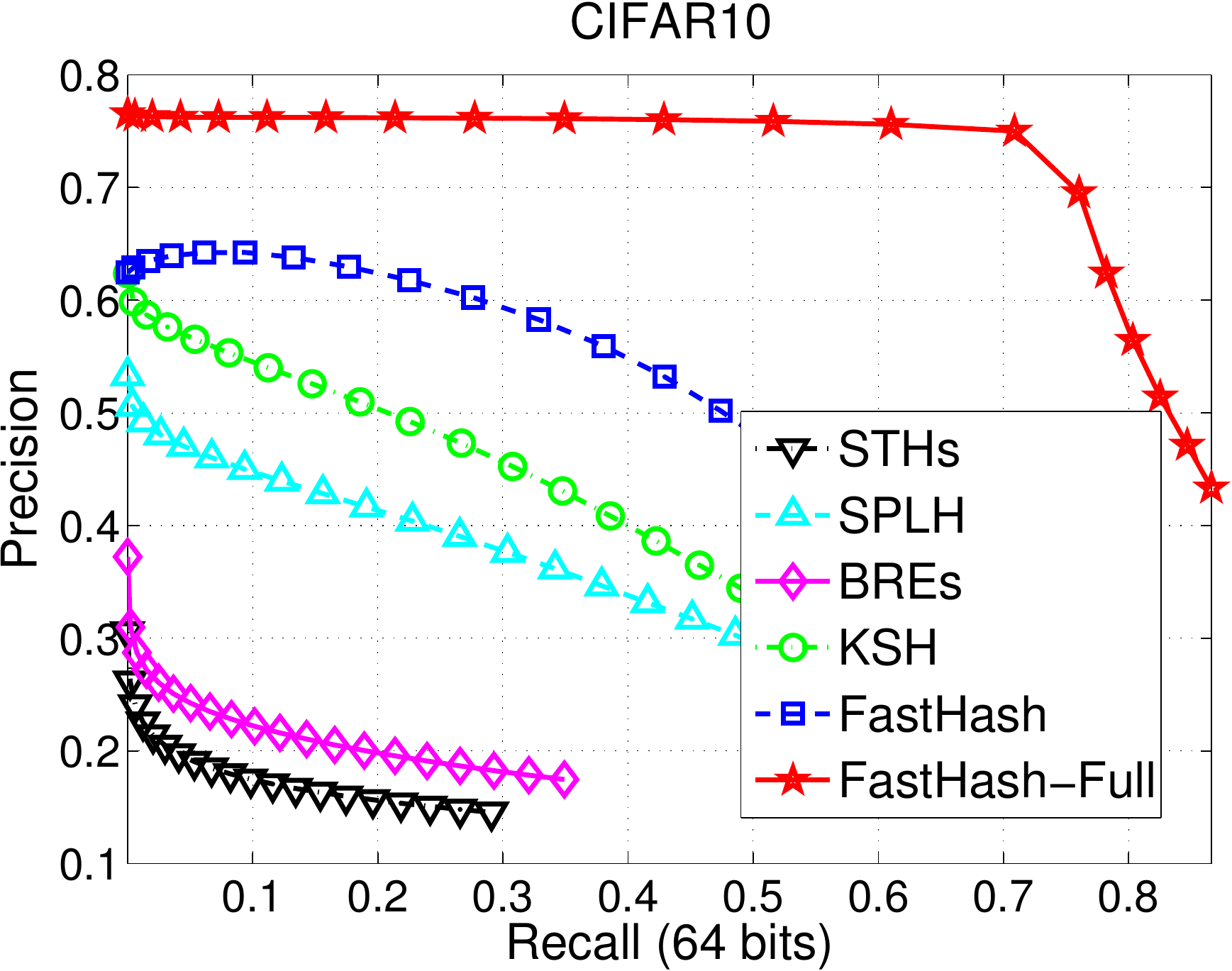}
   \includegraphics[width=.245\linewidth]{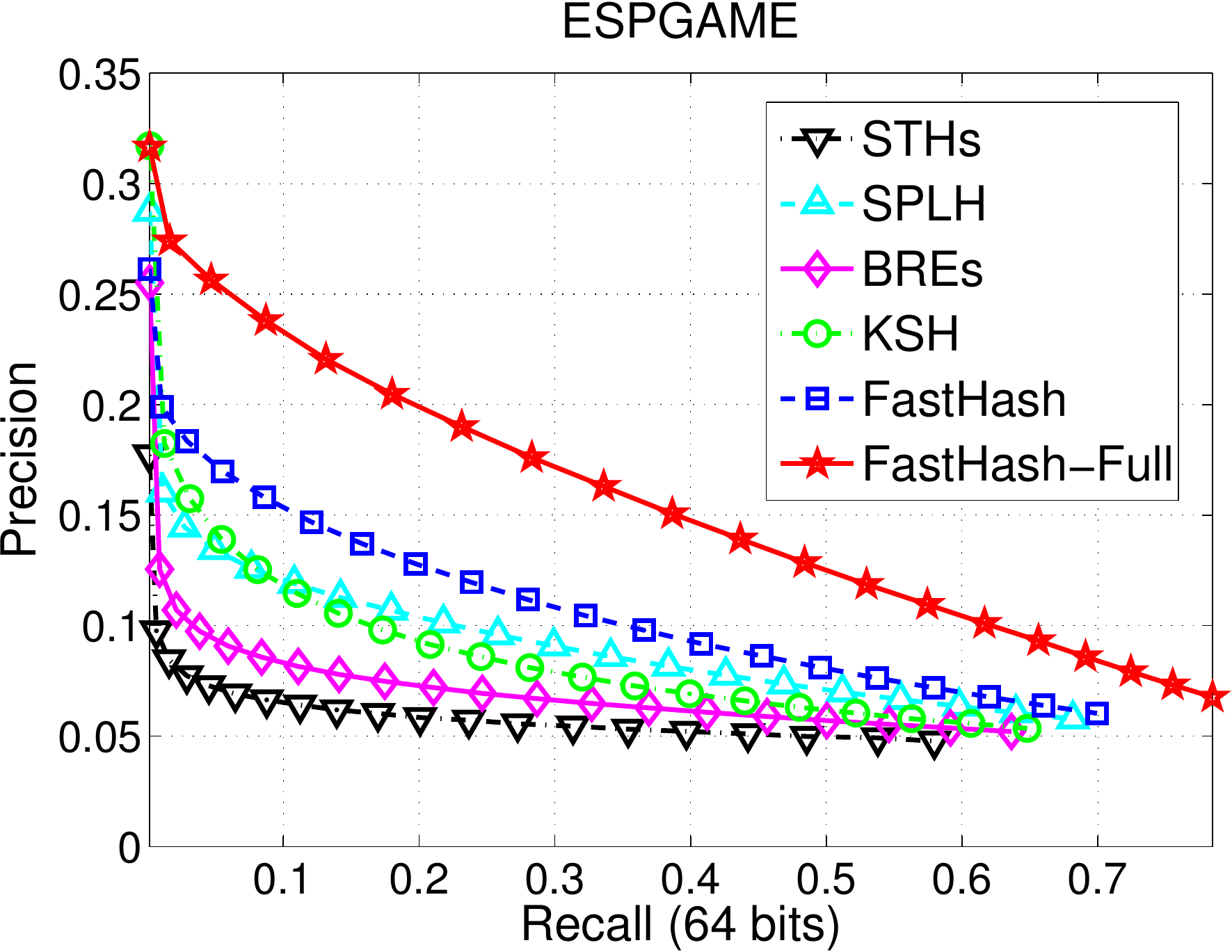}
   \includegraphics[width=.245\linewidth]{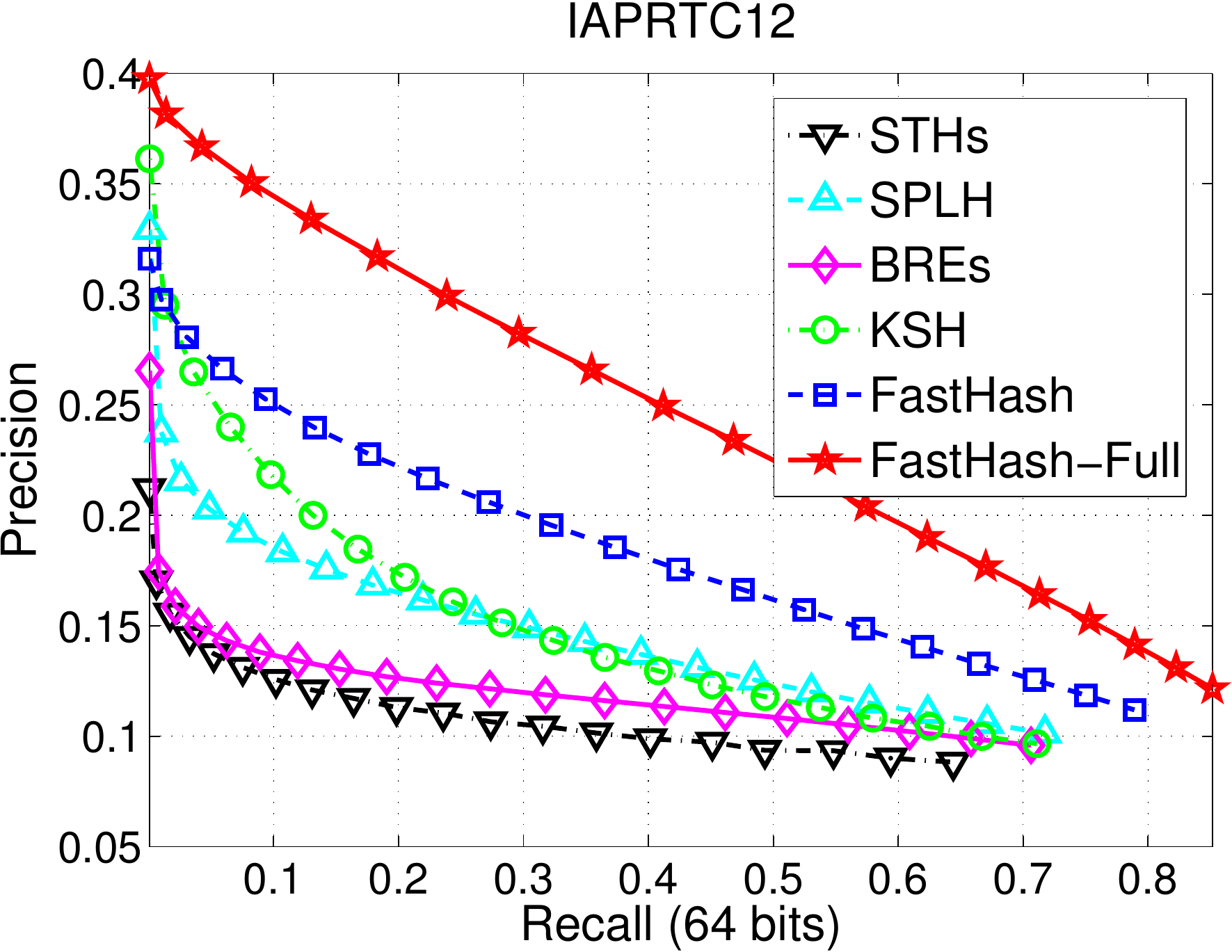}
   \includegraphics[width=.245\linewidth]{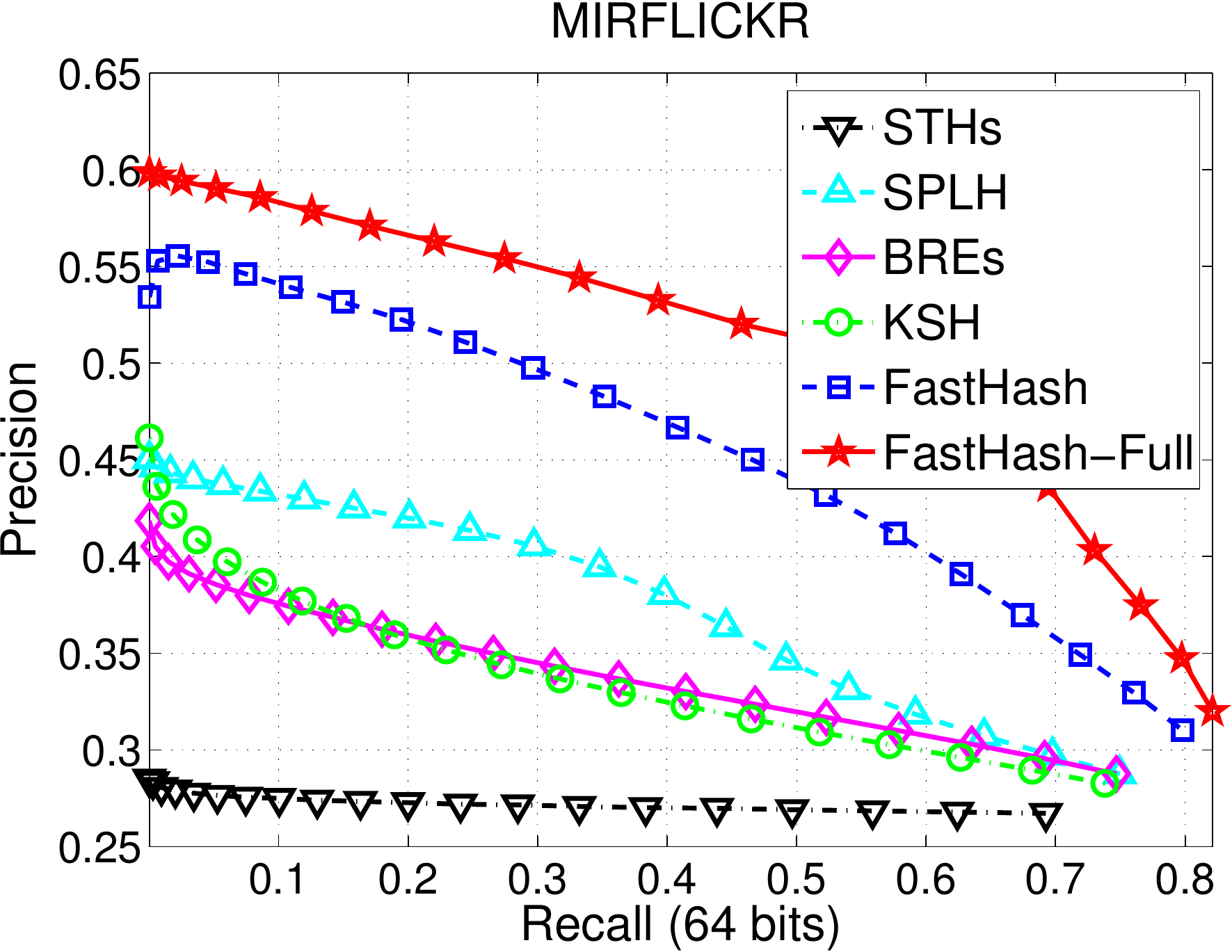}

   \includegraphics[width=.245\linewidth]{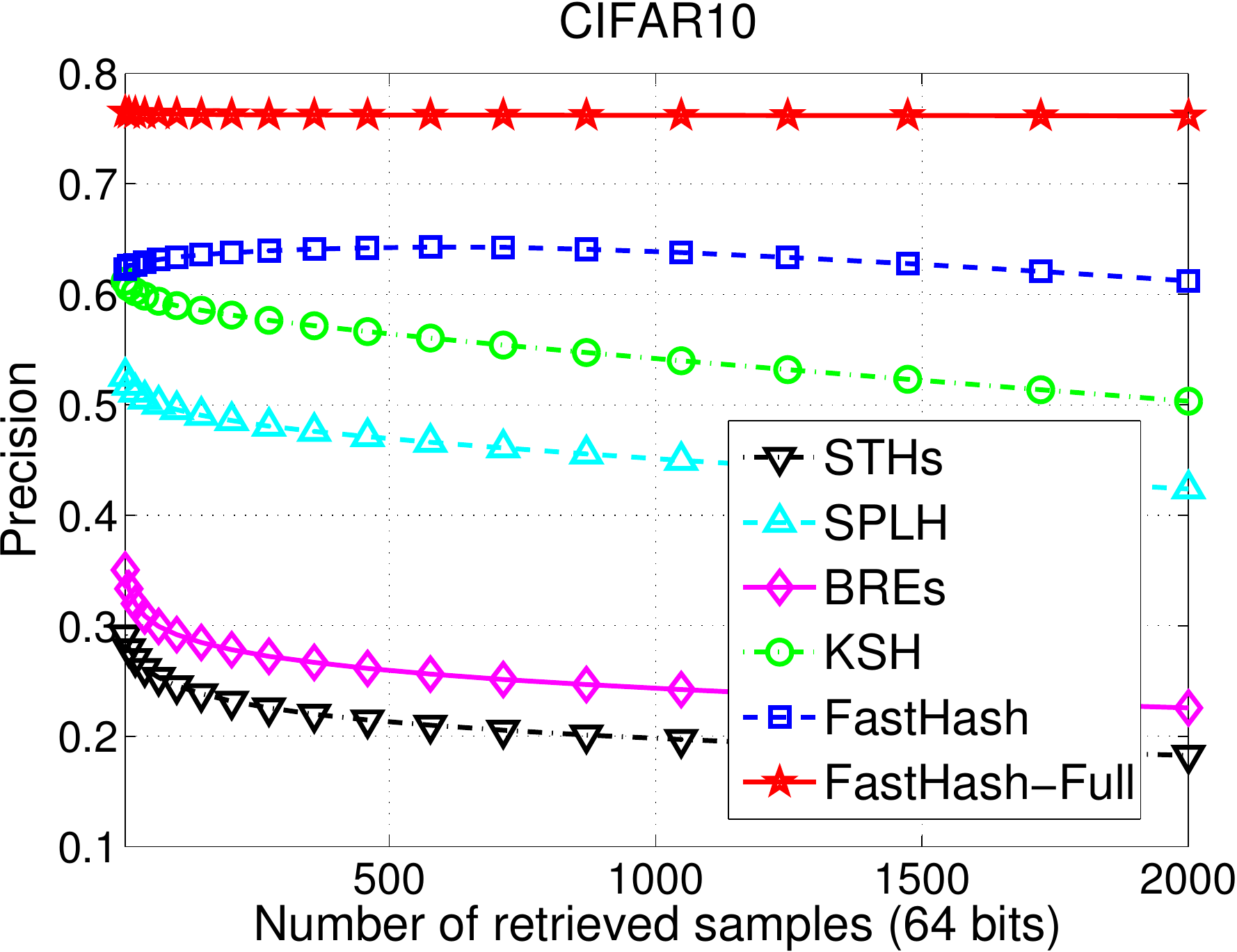}
   \includegraphics[width=.245\linewidth]{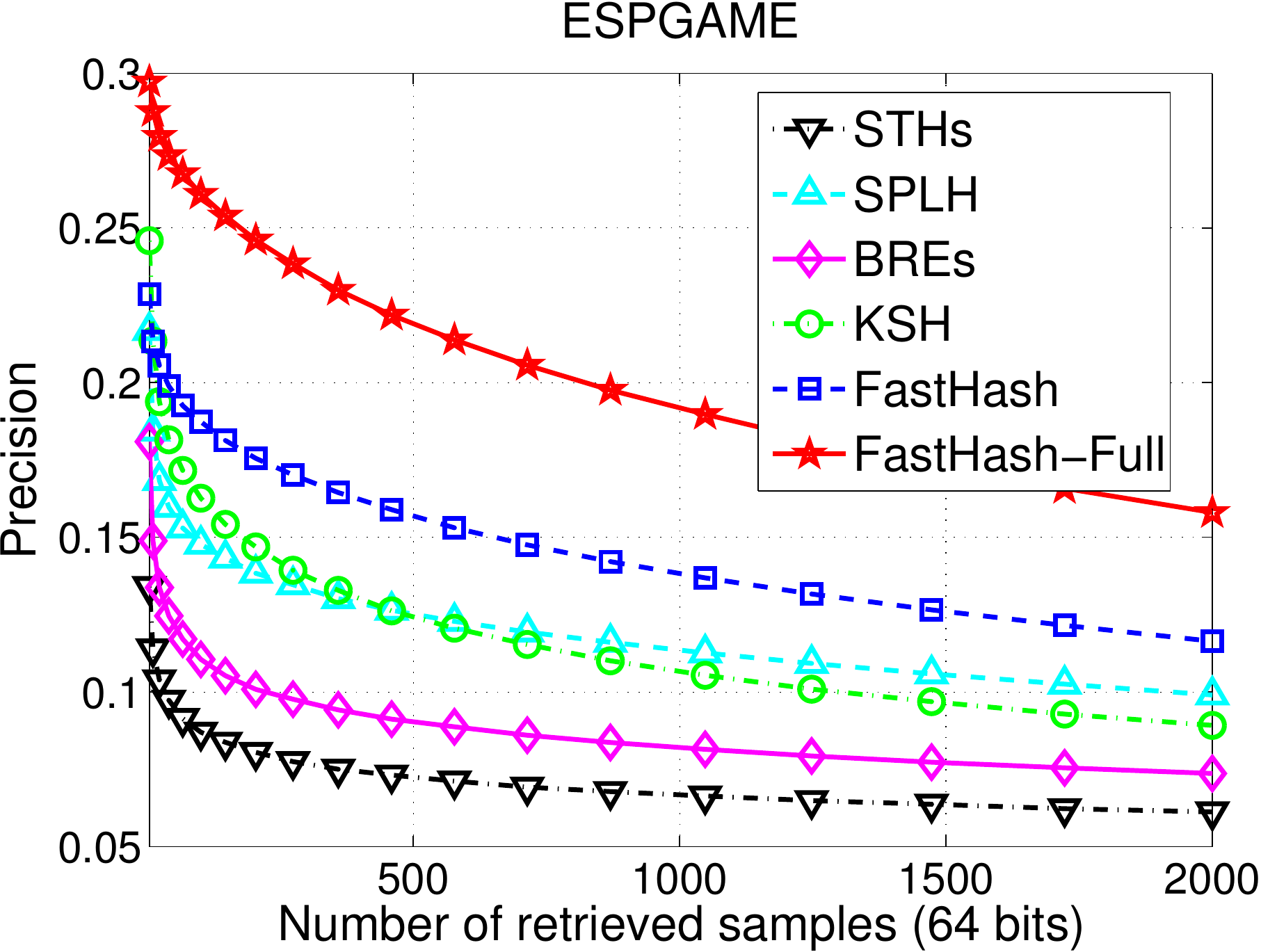}
   \includegraphics[width=.245\linewidth]{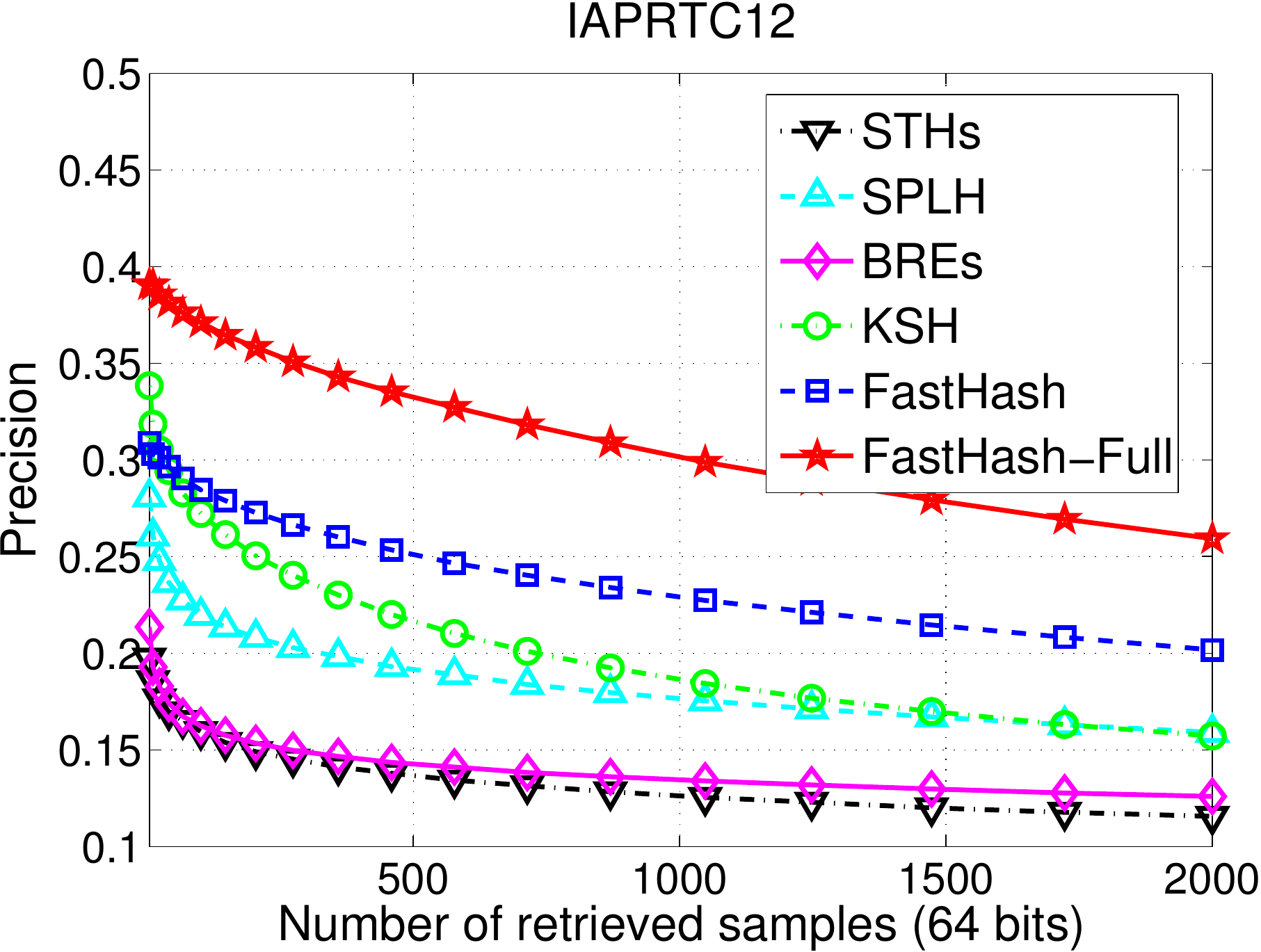}
   \includegraphics[width=.245\linewidth]{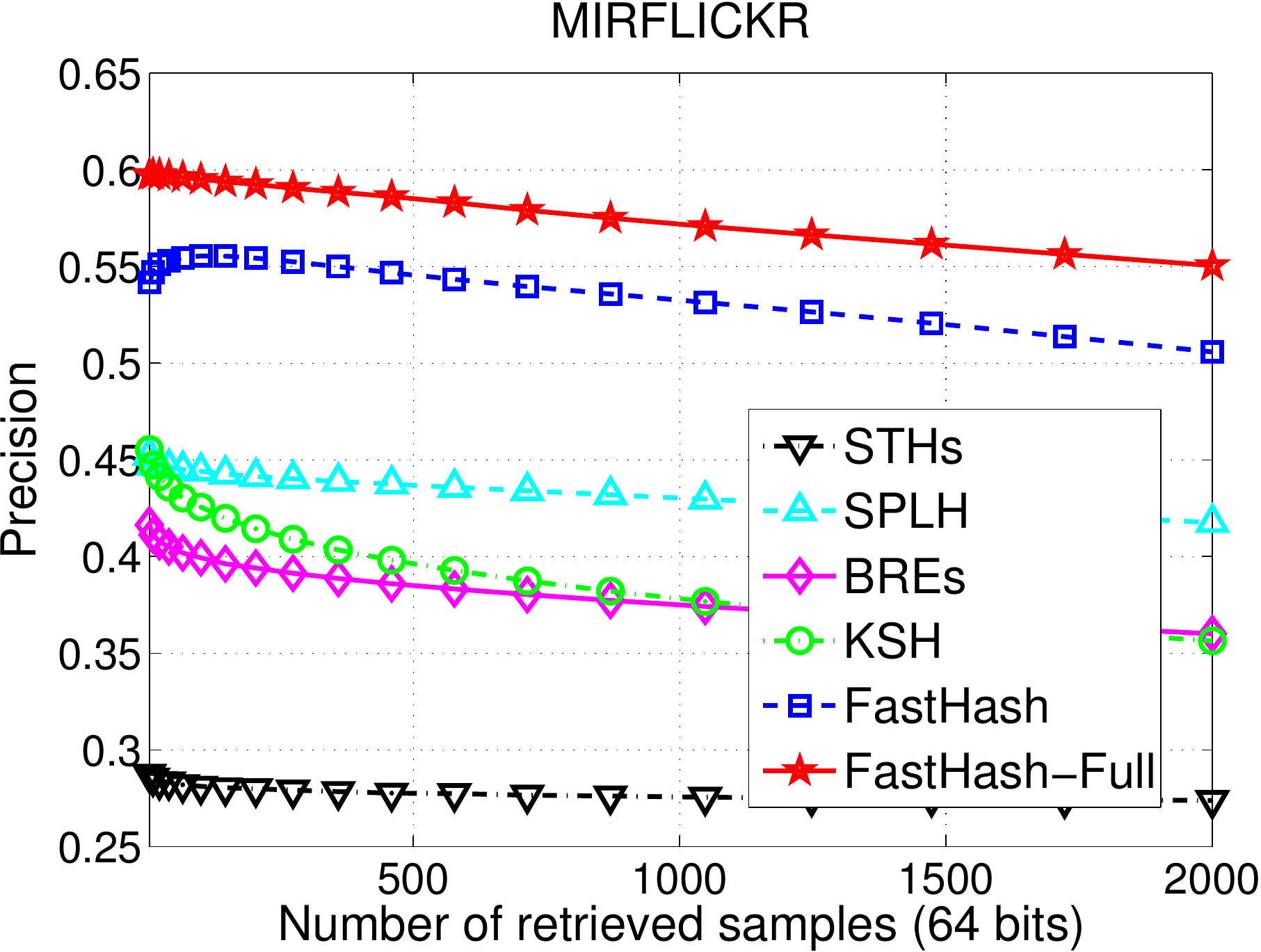}

    \caption{Results on high-dimensional codebook features.
    Our \fasth significantly outperform others.
    }
    \label{fh-fig:features}
\end{figure*}

\subsection{Comparison with KSH}
KSH \cite{KSH} has been shown to outperform many state-of-the-art comparators.
Here we evaluate our method using the KSH loss and compare against the original KSH method
on  high-dimensional codebook features.
KSH employs a simple kernel technique by predefining a set of support vectors then learning linear weights for each hash function.
For our method, we use boosted decision trees as hash functions.
KSH is trained on a sampled set of $5000$ examples, and
the number of support vectors for KSH is varied from $300$ to $3000$.
The results are summarized in Table \ref{fh-tab:ksh},
which shows that increasing the number of support vectors consistently improves the retrieval precision of KSH.
However, even on this small training set, including more support vectors will dramatically increase the training time and binary encoding time of KSH.
We have run our \fasth both on the same sampled training set and the whole training set (labeled as FastHash-Full)
in order to show
that our method can be efficiently trained on the whole dataset.
Our
\fasth and FastHash-Full outperform KSH by a large margin in terms of both training speed and retrieval precision.
It also shows that the decision tree hash functions in \fasth are much more efficient for testing (binary encoding) than the kernel function in KSH.
\fasth is orders of magnitude faster than KSH in training,
and thus much better suited to large training sets and high-dimensional data.
We also show
the precision curves of top-K retrieved examples
in Figure \ref{fh-fig:ksh}.
The number after ``KSH'' is the number of support vectors.
Besides high-dimensional features, we also compare with KSH on the low-dimensional GIST feature,
and \fasth also significantly performs better; see Table \ref{fh-tab:features} for details.
Some retrieval examples of our method are shown in Figure \ref{fh-fig:example_cifar}.

\subsection{Evaluation on different features}
We evaluate our method both on the low-dimensional ($320$ or $512$) GIST features and the
high-dimensional ($11200$) codebook features.
Several state-of-the-art supervised methods are included in this comparison:
KSH \cite{KSH}, Supervised Self-Taught Hashing (STHs) \cite{zhang2010self}, and
Semi-supervised Hashing (SPLH) \cite{wang2010semi}.
Comparison methods are trained on a sampled training set ($5000$ examples).
Results are presented in Table \ref{fh-tab:features}.
The codebook features consistently show better results than the GIST features.
The results for codebook features are also plotted in Figure \ref{fh-fig:features}.
It shows that the competing methods can be efficiently trained on the GIST features. However, when applied to high dimensional features, even on a small training set ($5000$), their training time dramatically increase.
It is very difficult to train these methods on the whole training set.
The training time of KSH mainly depends on the number of support vectors ($3000$ is used here).
We run our \fasth on the same sampled training set ($5000$ examples) and the whole training set (labeled as FastHash-Full).
Results show that \fasth can be efficiently trained on the whole dataset.
\fasth significantly outperform others both in GIST and codebook features.
The training of \fasth is also orders of magnitudes faster than others on the high-dimensional codebook features.

\begin{table}[t]
\caption{Results of methods with dimension reduction.
Our \fasth significantly outperforms others.}
\centering
 \resizebox{.95\linewidth}{!}
  {
  \begin{tabular}{ l l | l l c c}
  \hline\hline
Method  &\# Train &Train time &Test time  &Precision  &MAP\\
\hline
\multicolumn{6}{  c }{CIFAR10} \\ \hdashline
PCA+KSH&50000&$-$ &$-$  &$-$  &$-$\\
PCA+SPLH  &50000  &25984  &18 &0.482  &0.388\\
PCA+STHs  &50000  &7980 &18 &0.287  &0.200\\
CCA+ITQ&50000&1055  &7  &0.676  &0.642\\
\best FastH &50000  &1794&21  &\best 0.763  &\best 0.775  \\
\hline
\multicolumn{6}{  c }{IAPRTC12} \\ \hdashline
PCA+KSH &17665  &55031  &11 &0.082  &0.103\\
PCA+SPLH  &17665  &1855 &7  &0.239  &0.169\\
PCA+STHs  &17665  &2463 &7  &0.174  &0.126\\
CCA+ITQ&17665&804 &3  &0.332  &0.198\\
\best FastH &17665  &620  &9  &\best 0.371  &\best 0.276\\
\hline
\multicolumn{6}{  c }{ESPGAME} \\ \hdashline
PCA+KSH &18689  &55714  &11 &0.141  &0.084\\
PCA+SPLH  &18689  &2409 &7  &0.153  &0.103\\
PCA+STHs  &18689  &2777 &7  &0.098  &0.069\\
CCA+ITQ&18689&814 &3  &0.216  &0.131\\
\best FastH &18689  &663  &9  &\best 0.261  &\best 0.189\\
\hline
\multicolumn{6}{  c }{MIRFLICKR} \\ \hdashline
PCA+KSH &12500  &54260  &8  &0.384  &0.313\\
PCA+SPLH  &12500  &1054 &5  &0.445  &0.391\\
PCA+STHs  &12500  &1768 &5  &0.347  &0.301\\
CCA+ITQ &12500&699  &3  &0.519  &0.408\\
\best FastH &12500  &509  &7  &\best 0.595  &\best 0.558\\

  \hline \hline
  \end{tabular}
  }
\label{fh-tab:pca}
\end{table}

\subsection{Comparison with dimension reduction}
One possible way to reduce the training cost on high-dimensional data is to apply dimension reduction.
For the methods: KSH, SPLH and STHs, we thus reduce the original 11200-dimensional codebook features to 500 dimensions by applying PCA.
We also compare to CCA+ITQ \cite{gong2012iterative} which combines ITQ with supervised dimensional reduction.
Our \fasth still uses the original high-dimensional features.
The result is summarized in Table \ref{fh-tab:pca}.
After dimension reduction, most comparison methods can be trained on the whole training set within 24 hours
(except KSH on CIFAR10).
However it still much slower than our \fasth.
Our \fasth also performs significantly better on retrieval precision.
Learning decision tree hash functions in \fasth actually perform feature selection and hash function learning at the same time, which shows much better performance than other hashing methods with dimensional reduction.

\begin{table}[t]
\caption{Performance of our \fasth on more features ($22400$ dimensions) and more bits ($1024$ bits).
The training and binary coding time (test time) of \fasth is only linearly increased with the bit length.}
\centering
 \resizebox{.95\linewidth}{!}
  {
  \begin{tabular}{ l l l | l l c c}
    \hline \hline
Bits   &\#Train  &Features &Train time  &Test time  &Precision  &MAP  \\
\hline
\multicolumn{7}{ c}{CIFAR10}\\ \hdashline
64  &50000  &11200  &1794 &21 &0.763  &0.775  \\
256 &50000  &22400  &5588 &71 &0.794  &0.814  \\
1024  &50000  &22400  &22687  &282  &0.803  &0.826  \\
\hline
\multicolumn{7}{ c}{IAPRTC12} \\ \hdashline
64  &17665  &11200  &320  &9  &0.371  &0.276   \\
256 &17665  &22400  &1987 &33 &0.439  &0.314    \\
1024  &17665  &22400  &7432 &134  &0.483  &0.338   \\
\hline
\multicolumn{7}{ c}{ESPGAME} \\ \hdashline
64  &18689  &11200  &663  &9  &0.261  &0.189\\
256 &18689  &22400  &1912 &34 &0.329  &0.233\\
1024  &18689  &22400  &7689 &139  &0.373  &0.257\\
\hline
\multicolumn{7}{ c}{MIRFLICKR} \\ \hdashline
64  &12500  &11200  &509  &7  &0.595  &0.558\\
256 &12500  &22400  &1560 &28 &0.612  &0.567\\
1024  &12500  &22400  &6418 &105  &0.628  &0.576\\
\hline\hline
  \end{tabular}
  }
\label{fh-tab:long_bits}
\end{table}

\subsection{More features and more bits}
We increase the codebook size to $1600$ for generating higher dimensional features ($22400$ dimensions) and run up to $1024$ bits.
Table \ref{fh-tab:long_bits}
shows that \fasth can be efficiently trained on high-dimensional features with large bit length. The training and binary coding time (test time) of \fasth increases only linearly  with bit length. The retrieval result is improved when the bit length is increased.

\begin{table}[t]
\caption{Comparison of spectral method and the proposed Block GraphCut (Block-GC)
for binary code inference.
Block-GC achieves lower objective value and takes less inference time,
  thus performs much better.
}
\centering
 \resizebox{.95\linewidth}{!}
  {
  \begin{tabular}{ l | l c | c c}
  \hline\hline
Step-1 methods  &\#train &Block Size  & Time (s)  &Objective\\
\hline
\multicolumn{5}{  c }{SUN397} \\ \hdashline
Spectral  &100417 &N/A  &5281 &0.7524\\
Block-GC-1  &100417 &1  &\best 298  &0.6341\\
\best Block-GC &100417 &253  &2239 &\best 0.5608\\
\hline
\multicolumn{5}{  c }{CIFAR10} \\ \hdashline
Spectral &50000  &N/A  &1363 &0.4912\\
Block-GC-1 &50000  &1  &\best 158  &0.5338\\
\best Block-GC &50000  &5000 &788  &\best 0.4158\\
\hline
\multicolumn{5}{  c }{IAPRTC12} \\ \hdashline
Spectral  &17665  &N/A  &426  &0.7237\\
Block-GC-1   &17665  &1  &\best 43 &0.7316\\
\best Block-GC  &17665  &316  &70 &\best 0.7095\\
\hline
\multicolumn{5}{  c }{ESPGAME} \\ \hdashline
Spectral &18689  &N/A  &480  &0.7373\\
Block-GC-1 &18689  &1  &\best 45 &0.7527\\
\best Block-GC &18689  &336  &72 &\best 0.7231\\
\hline
\multicolumn{5}{  c }{MIRFLICKR} \\ \hdashline
Spectral &12500  &N/A  &125  &0.5718\\
Block-GC-1 &12500  &1  &\best 28 &0.5851\\
\best Block-GC &12500  &295  &40 &\best 0.5449\\
  \hline \hline
  \end{tabular}
  }
\label{fh-tab:tsh_step1}
\end{table}

\begin{table}[t]
\caption{Comparison of combinations of hash functions and binary inference methods.
Decision tree hash functions perform much better than linear SVM (LSVM) hash functions.
The proposed Block GraphCut (Block-GC)
performs much better than the spectral method.
}
\centering
\resizebox{0.95\linewidth}{!}
  {
  \begin{tabular}{ l l | c c c}
    \hline \hline
Step-1 method   &Step-2 method &Precision  &MAP  &Prec-Recall  \\
\hline
\multicolumn{5}{ c}{CIFAR10} \\ \hdashline
\best Block-GC	&TREE	&\best 0.763	&\best 0.775	&\best 0.605\\
Spectral	&TREE	&0.731	&0.695	&0.501\\
Block-GC	&LSVM	&0.669	&0.621	&0.435\\
Spectral	&LSVM	&0.624	&0.512	&0.322\\
\hline
\multicolumn{5}{ c}{IAPRTC12} \\ \hdashline
\best Block-GC	&TREE	&\best 0.371	&\best 0.276	&\best 0.210\\
Spectral	&TREE	&0.355	&0.265	&0.201\\
Block-GC	&LSVM	&0.327	&0.238	&0.186\\
Spectral	&LSVM	&0.275	&0.207	&0.160\\
\hline
\multicolumn{5}{ c}{ESPGAME} \\ \hdashline
\best Block-GC	&TREE	&\best 0.261	&\best 0.189	&\best 0.126\\
Spectral	&TREE	&0.249	&0.183	&0.123\\
Block-GC	&LSVM	&0.227	&0.157	&0.109\\
Spectral	&LSVM	&0.183	&0.133	&0.093\\
\hline
\multicolumn{5}{ c}{MIRFLICKR} \\ \hdashline
\best Block-GC	&TREE	&\best 0.595	&\best 0.558	&\best 0.420\\
Spectral	&TREE	&0.584	&0.551	&0.413\\
Block-GC	&LSVM	&0.536	&0.498	&0.344\\
Spectral	&LSVM	&0.489	&0.466	&0.319\\
  \hline \hline
  \end{tabular}
  }
\label{fh-tab:tsh_step2}
\end{table}

\begin{figure*}[t]
    \centering

  \includegraphics[width=.245\linewidth]{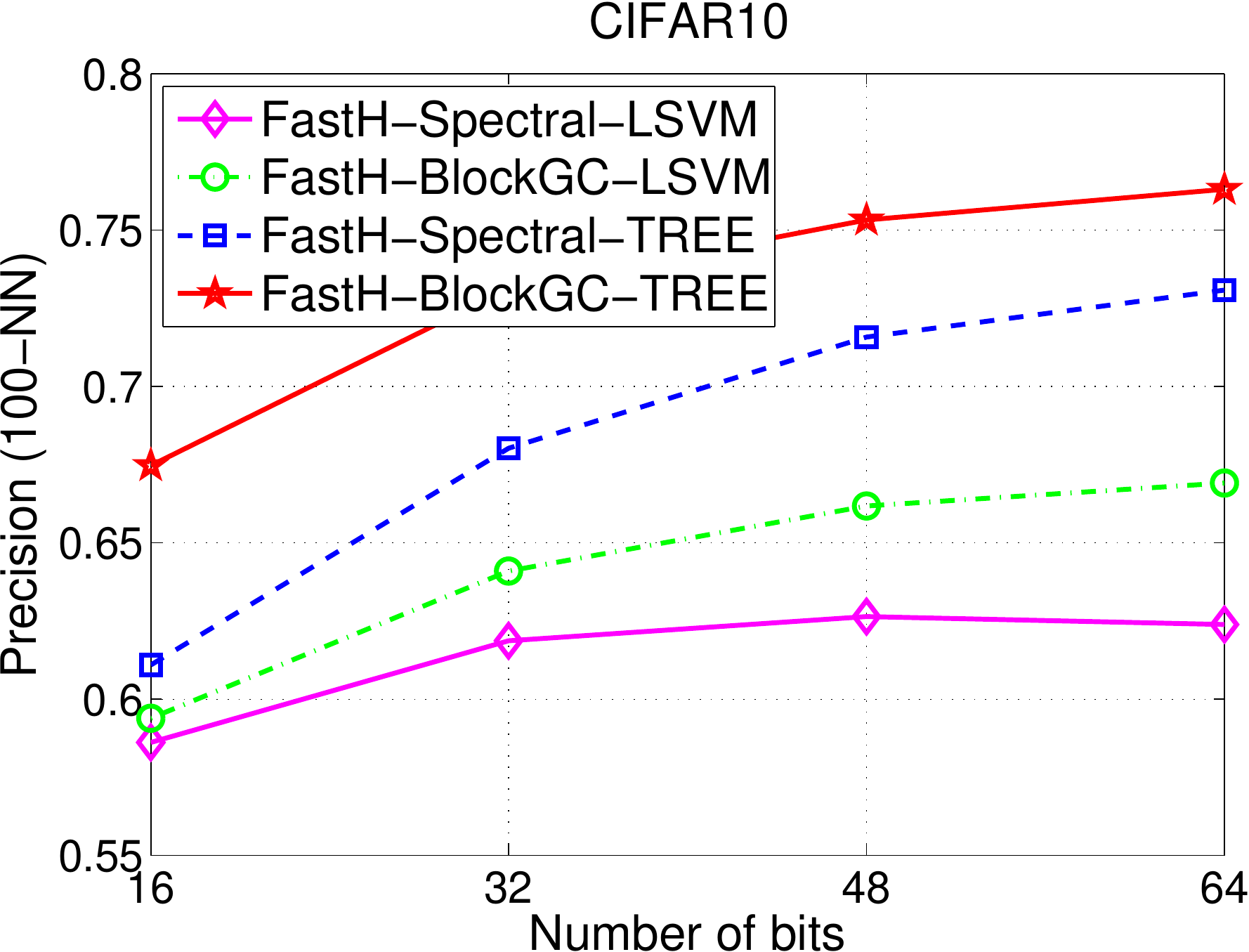}
   \includegraphics[width=.245\linewidth]{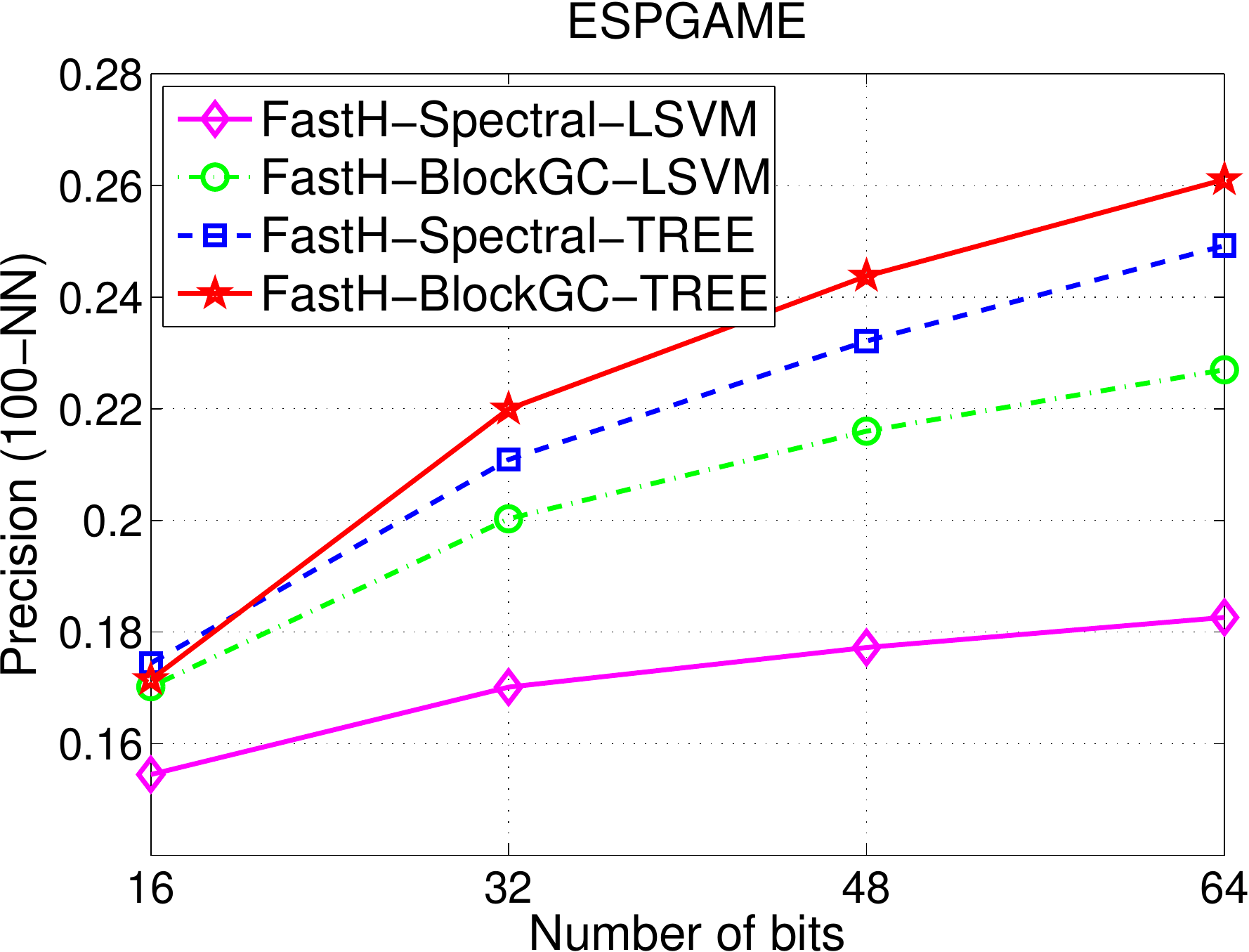}
   \includegraphics[width=.245\linewidth]{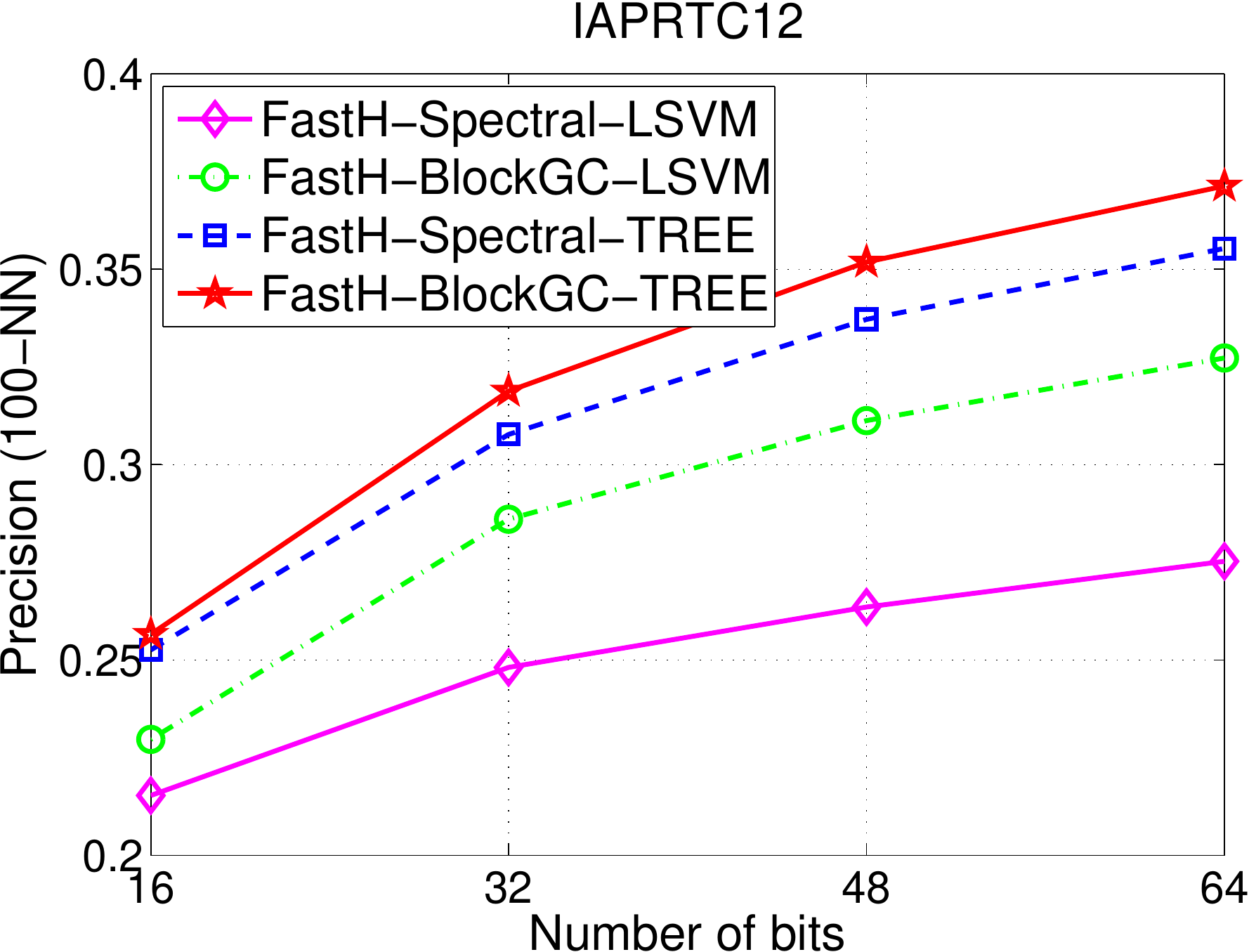}
   \includegraphics[width=.245\linewidth]{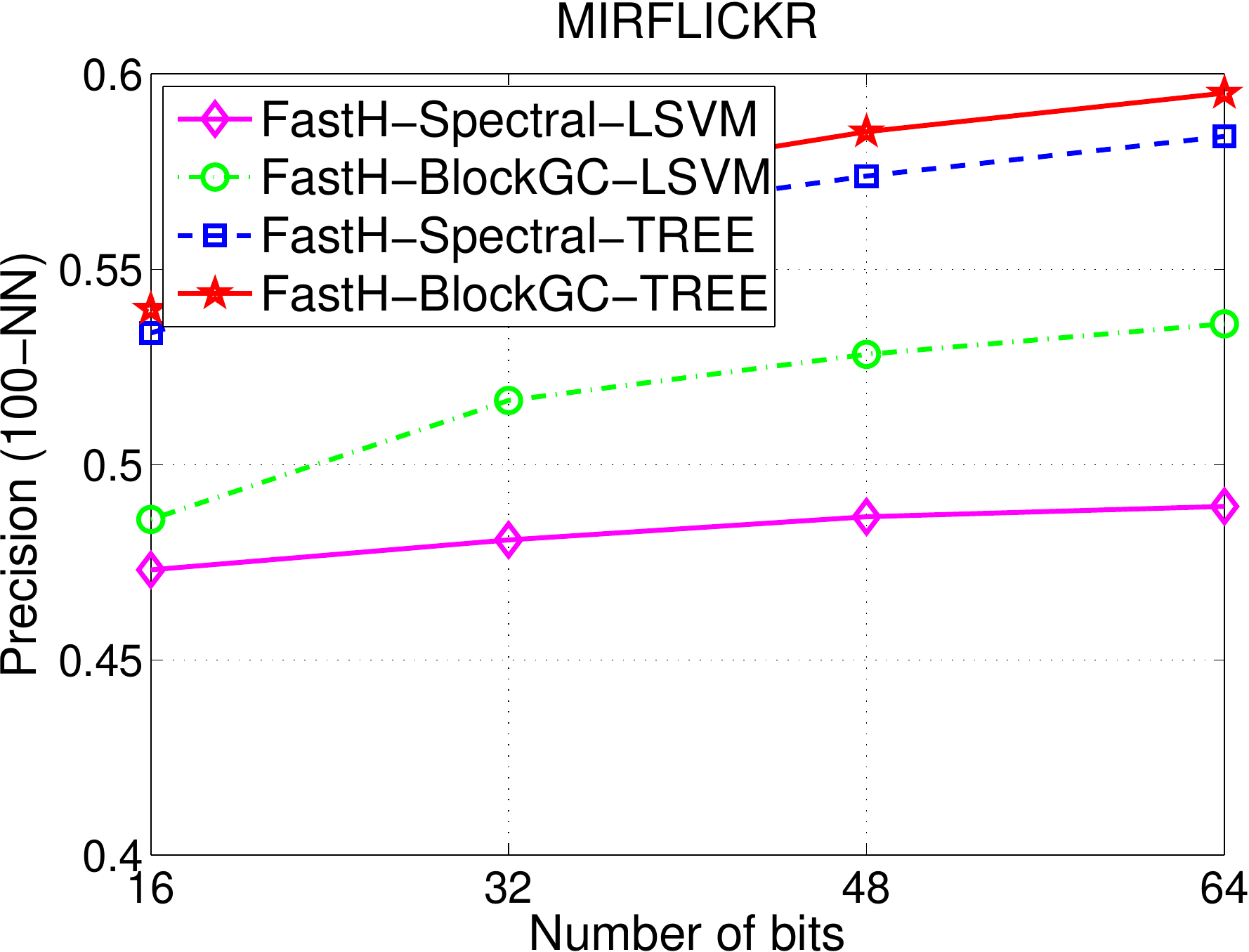}

    \caption{Comparison of combinations of hash functions and binary inference methods.
    Decision tree hash functions perform much better than linear SVM.
The proposed Block-GC performs much better than the spectral method.
}
    \label{fh-fig:tsh}
\end{figure*}

\subsection{Binary code inference evaluation}

Here we evaluate different algorithms for solving the binary code inference problem which is involved in Step-1 of our learning process.
We compare the proposed Block GraphCut (Block-GC) with the simple spectral method.
The number of iterations of Block-GC is set to $2$, which is the same as that in other experiments.
Results are summarized in Table \ref{fh-tab:tsh_step1}.
We construct blocks using Algorithm \ref{fh-alg:block}.
The averaged block size is reported in the table.
We also evaluate a special case where the block size is set to 1 for Block-CG (labeled as Block-CG-1),
in which case Block-GC is reduced to the ICM (\cite{besag1986statistical, UGM}) method.
It shows that when the training set gets larger, the spectral method becomes slow.
The objective value shown in the table is divided by the number of defined pairwise relations.
Results show that the proposed Block-GC achieves much lower objective values and takes less inference time, and hence outperforms the spectral method.
The inference time for Block-CG increases only linearly with training set size.

Our method is able to incorporate different kinds of hash functions in Step-2.
Here we provide results comparing different combinations of
hash functions (Step-2) and binary code inference methods (Step-1).
We evaluate linear SVM and decision tree hash functions with
the spectral method and the proposed Block-GC.
Codebook features are used here.
Results are summarized in Table \ref{fh-tab:tsh_step2}.
We also plot the retrieval performance in Figure~\ref{fh-fig:tsh}.
As expected, decision tree hash functions perform much better than linear SVM hash functions,
and the proposed Block-GC performs much better than the spectral method,
which indicates that Block-GC is able to generate high quality binary codes.

\begin{figure*}
    \centering

   \includegraphics[width=.245\linewidth]{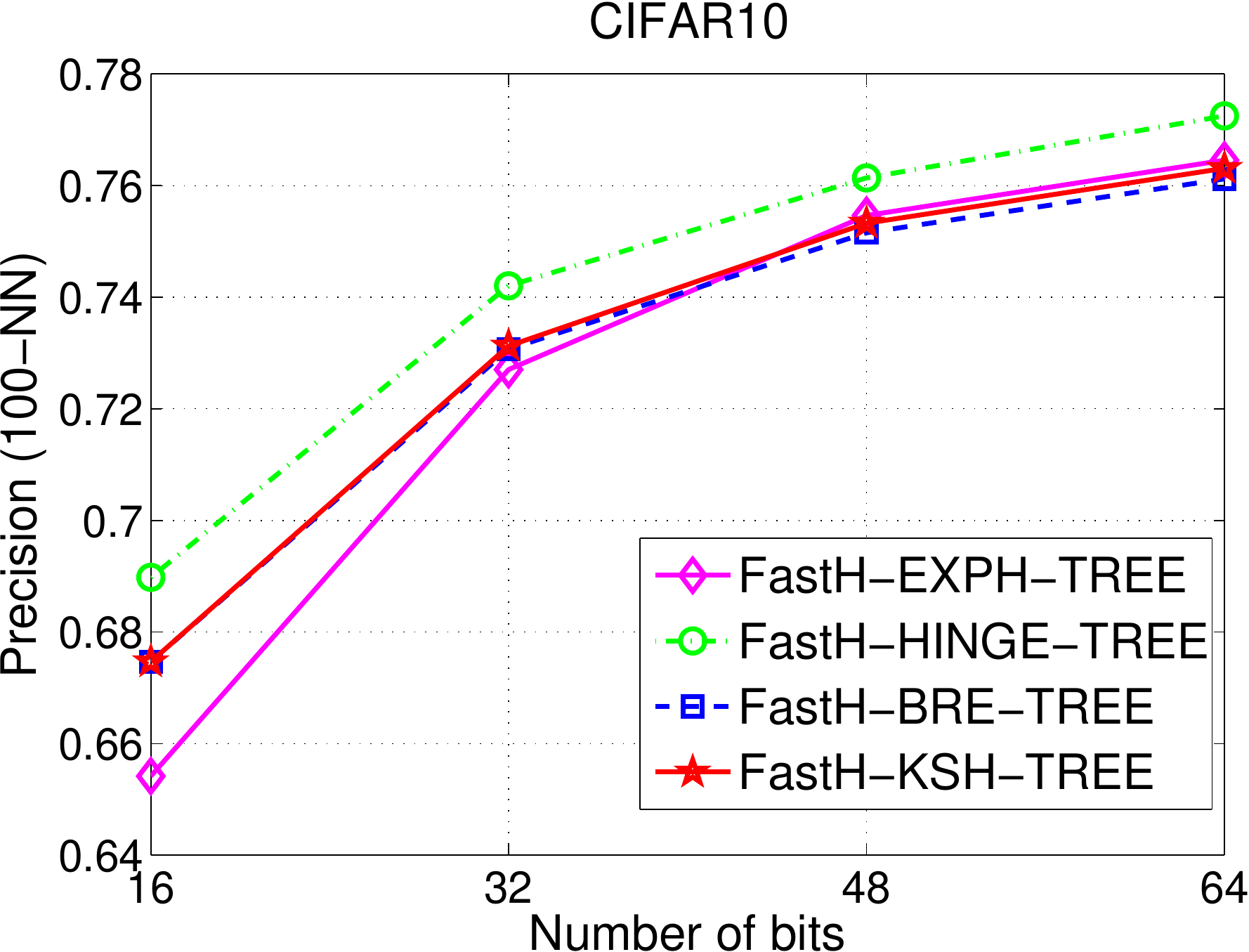}
   \includegraphics[width=.245\linewidth]{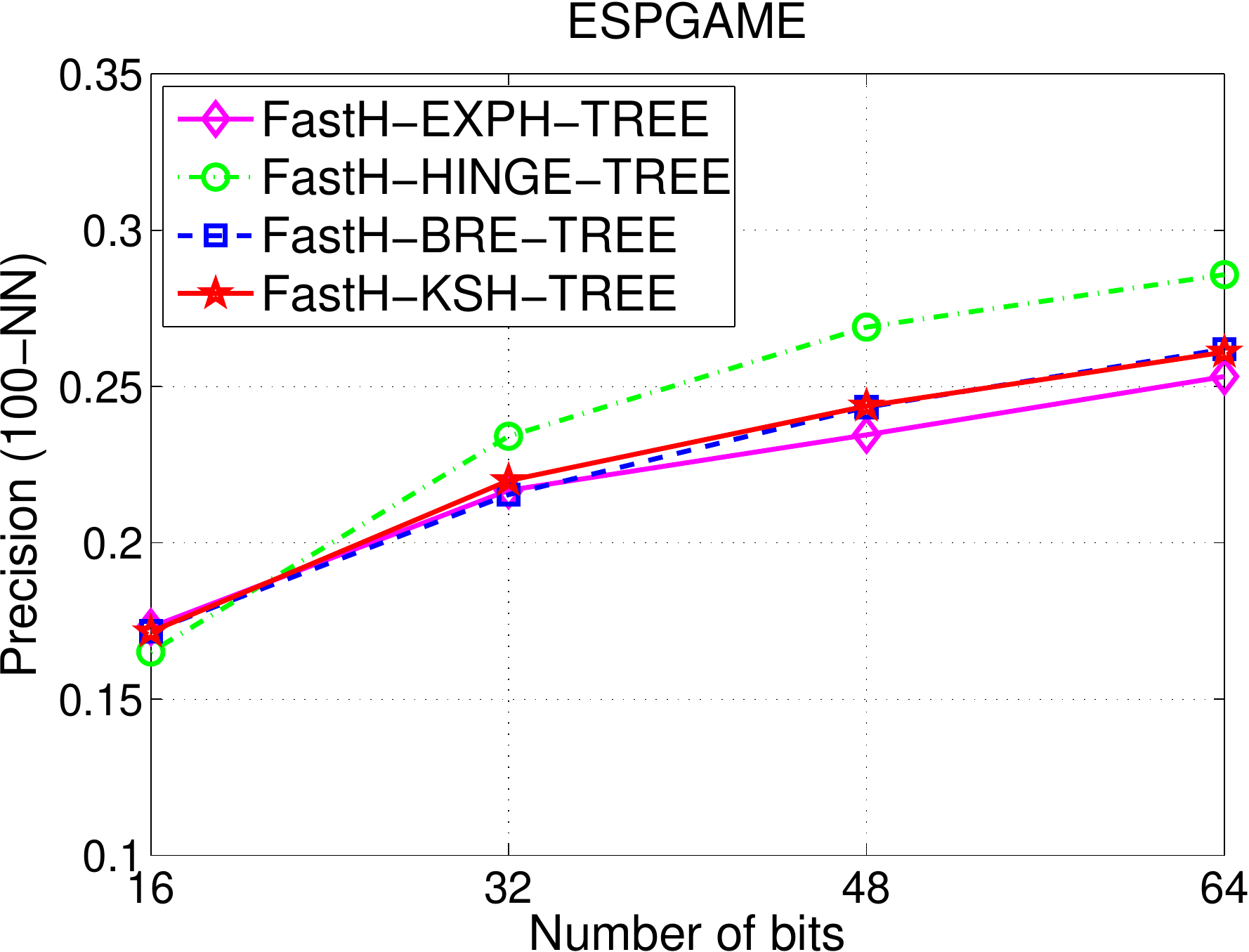}
   \includegraphics[width=.245\linewidth]{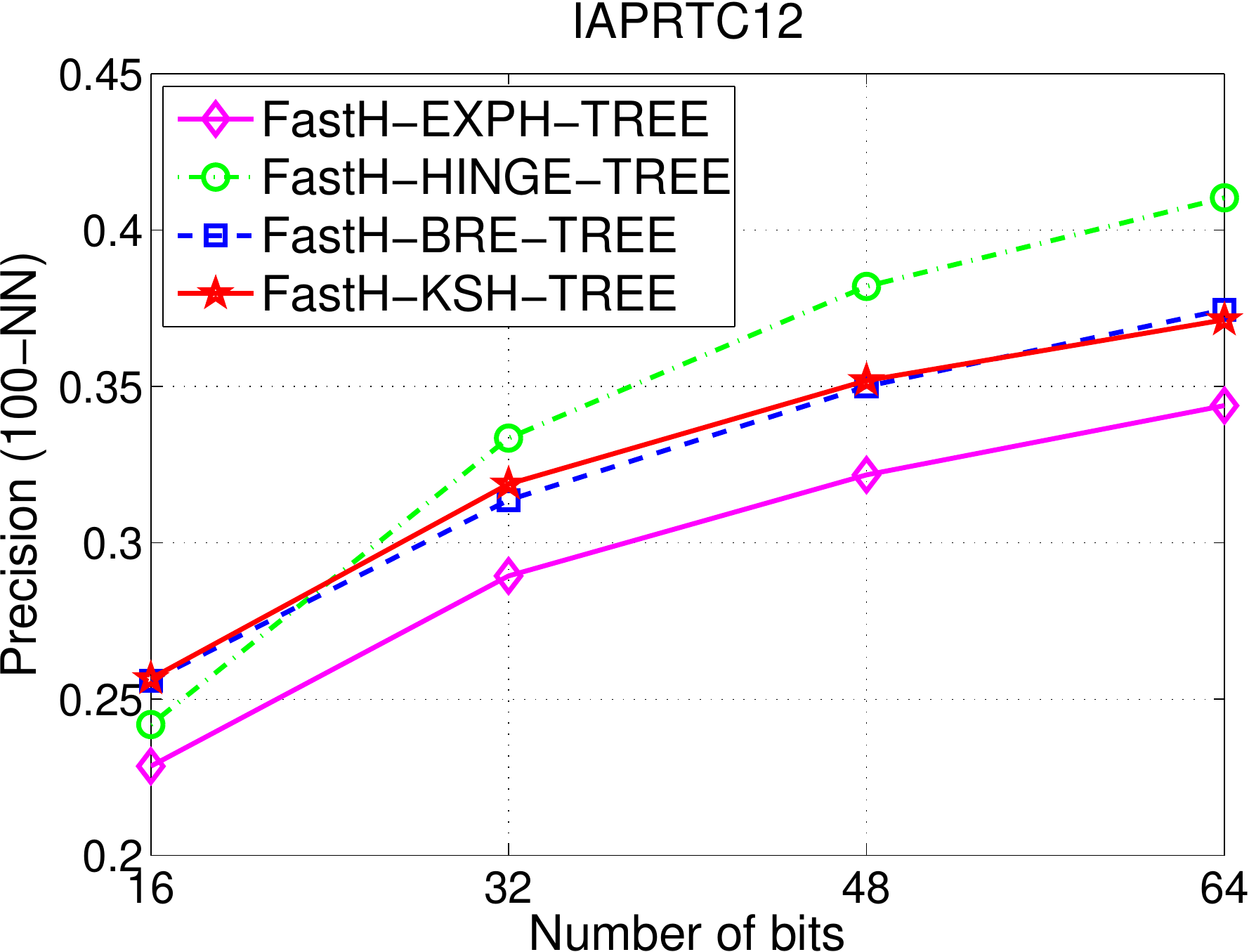}
   \includegraphics[width=.245\linewidth]{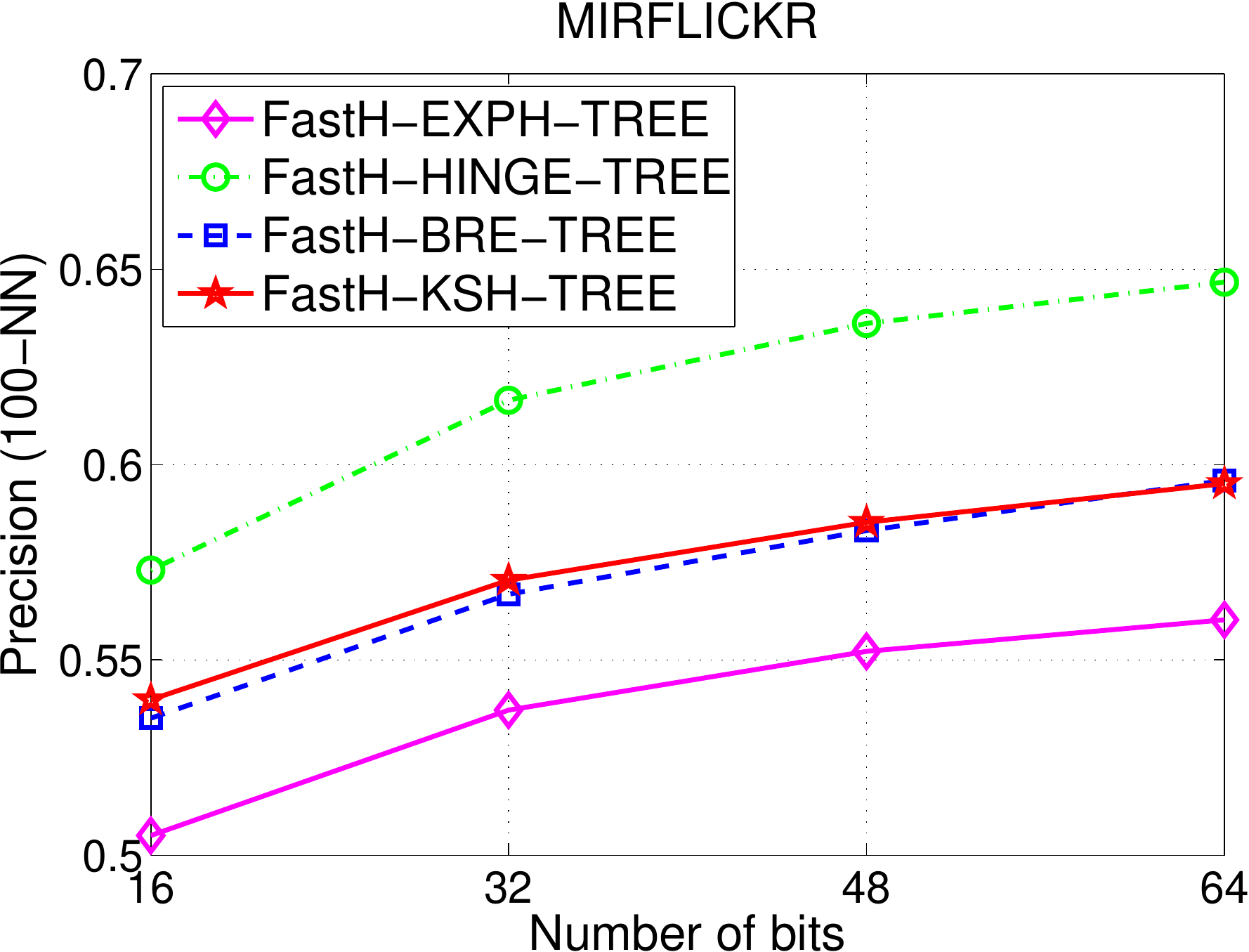}

    \caption{
    Comparison of using different loss functions with decision tree hash functions.
Using the Hinge loss (FastH-Hinge) achieves the best result.
}
    \label{fh-fig:loss}
\end{figure*}

\subsection{Using different loss functions}

Our method is able to incorporate different kinds of loss functions and hash functions.
Here we compare the performance of 4 kinds of loss function: KSH \eqref{eq:loss-ksh}, Hinge \eqref{eq:loss-hinge}, BRE \eqref{eq:loss-bre} and ExpH \eqref{eq:loss-exph}, combined with linear SVM \eqref{eq:hashfun_lsvm} and boosted decision tree \eqref{eq:hashfun_btree} hash functions.
Results are summarized in Table \ref{fh-tab:loss}.
It shows that the Hinge loss usually achieves the best performance,
and the remaining loss functions have similar performances.
The Hinge loss function in \eqref{eq:loss-hinge}
encourages the hamming distance of dissimilar pairs to be at least half of the bit length,
instead of unnecessarily pushing it to the maximum value.
It also shows that decision tree hash functions perform much better than linear SVM.
 We plot the performance of decision tree hash functions combined with different kinds of loss functions
 in Figure \ref{fh-fig:loss}.

\begin{table}[t]
\caption{Comparison of combinations of different loss functions and hash functions.
Using the Hinge loss achieves the best result. Decision tree hash functions perform much better than linear SVM hash functions.
}
\centering
\resizebox{0.9\linewidth}{!}
  {
  \begin{tabular}{ l l | c c c}
    \hline \hline
Loss   &Step-2 method &Precision  &MAP  &Prec-Recall  \\
\hline
\multicolumn{5}{ c}{CIFAR10} \\ \hdashline
FastH-KSH	&TREE	&0.763	&0.775	&0.605\\
FastH-BRE	&TREE	&0.761	&0.772	&0.602\\
\best FastH-HINGE	&TREE	&\best 0.773	&\best 0.780	&\best 0.613\\
FastH-EXPH	&TREE	&0.765	&0.774	&0.604\\
FastH-KSH	&LSVM	&0.669	&0.621	&0.435\\
FastH-BRE	&LSVM	&0.667	&0.619	&0.431\\
FastH-HINGE	&LSVM	&0.669	&0.604	&0.387\\
FastH-EXPH	&LSVM	&0.665	&0.619	&0.430\\
\hline
\multicolumn{5}{ c}{IAPRTC12} \\ \hdashline
FastH-KSH	&TREE	&0.371	&0.276	&0.210\\
FastH-BRE	&TREE	&0.375	&0.279	&0.213\\
\best FastH-HINGE	&TREE	&\best 0.410	&\best 0.295	&\best 0.234\\
FastH-EXPH	&TREE	&0.344	&0.268	&0.199\\
FastH-KSH	&LSVM	&0.327	&0.238	&0.186\\
FastH-BRE	&LSVM	&0.328	&0.237	&0.187\\
FastH-HINGE	&LSVM	&0.338	&0.247	&0.194\\
FastH-EXPH	&LSVM	&0.295	&0.225	&0.170\\
\hline
\multicolumn{5}{ c}{ESPGAME} \\ \hdashline
FastH-KSH	&TREE	&0.261	&0.189	&0.126\\
FastH-BRE	&TREE	&0.262	&0.189	&0.125\\
\best FastH-HINGE	&TREE	&\best 0.286	&\best 0.200	&\best 0.148\\
FastH-EXPH	&TREE	&0.253	&0.194	&0.124\\
FastH-KSH	&LSVM	&0.227	&0.157	&0.109\\
FastH-BRE	&LSVM	&0.231	&0.160	&0.111\\
FastH-HINGE	&LSVM	&0.225	&0.155	&0.109\\
FastH-EXPH	&LSVM	&0.216	&0.154	&0.104\\
\hline
\multicolumn{5}{ c}{MIRFLICKR} \\ \hdashline
FastH-KSH	&TREE	&0.595	&0.558	&0.420\\
FastH-BRE	&TREE	&0.596	&0.559	&0.420\\
\best FastH-HINGE	&TREE	&\best 0.647	&\best 0.592	&\best 0.457\\
FastH-EXPH	&TREE	&0.560	&0.543	&0.404\\
FastH-KSH	&LSVM	&0.536	&0.498	&0.344\\
FastH-BRE	&LSVM	&0.531	&0.494	&0.341\\
FastH-HINGE	&LSVM	&0.567	&0.522	&0.397\\
FastH-EXPH	&LSVM	&0.502	&0.471	&0.323\\
  \hline \hline
  \end{tabular}
  }
\label{fh-tab:loss}
\end{table}

\begin{figure*}[t]
    \centering

   \includegraphics[width=.245\linewidth]{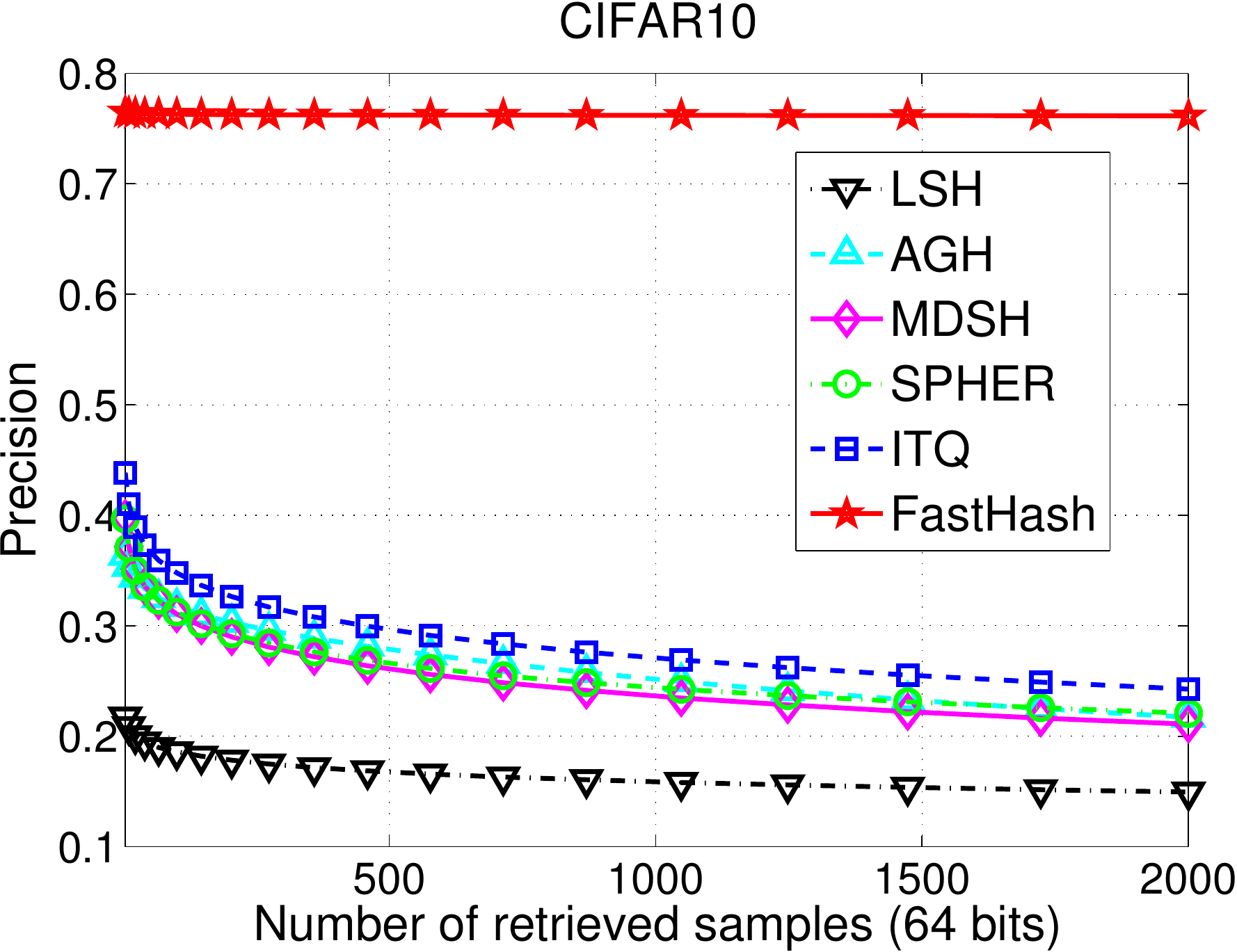}
   \includegraphics[width=.245\linewidth]{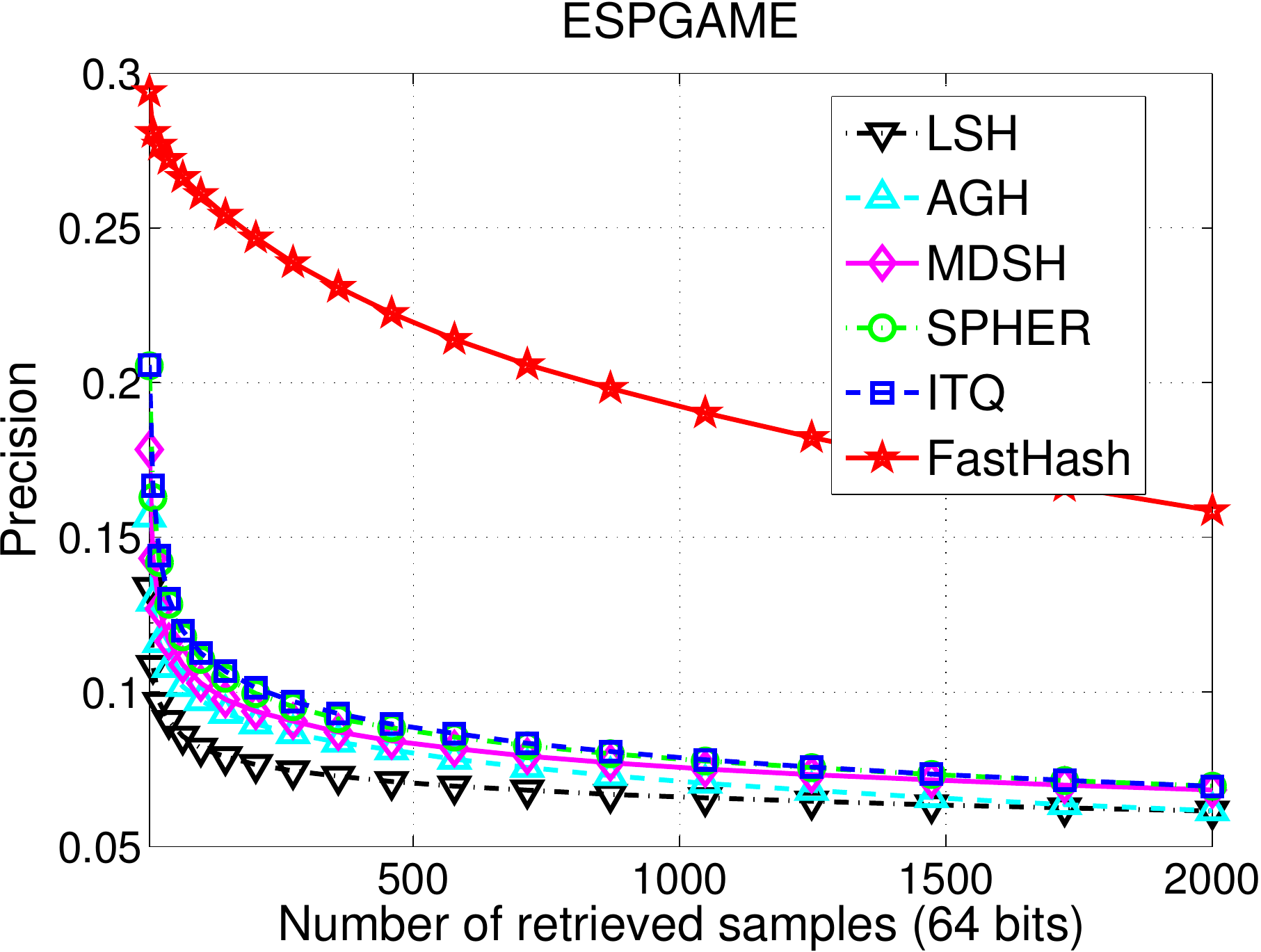}
   \includegraphics[width=.245\linewidth]{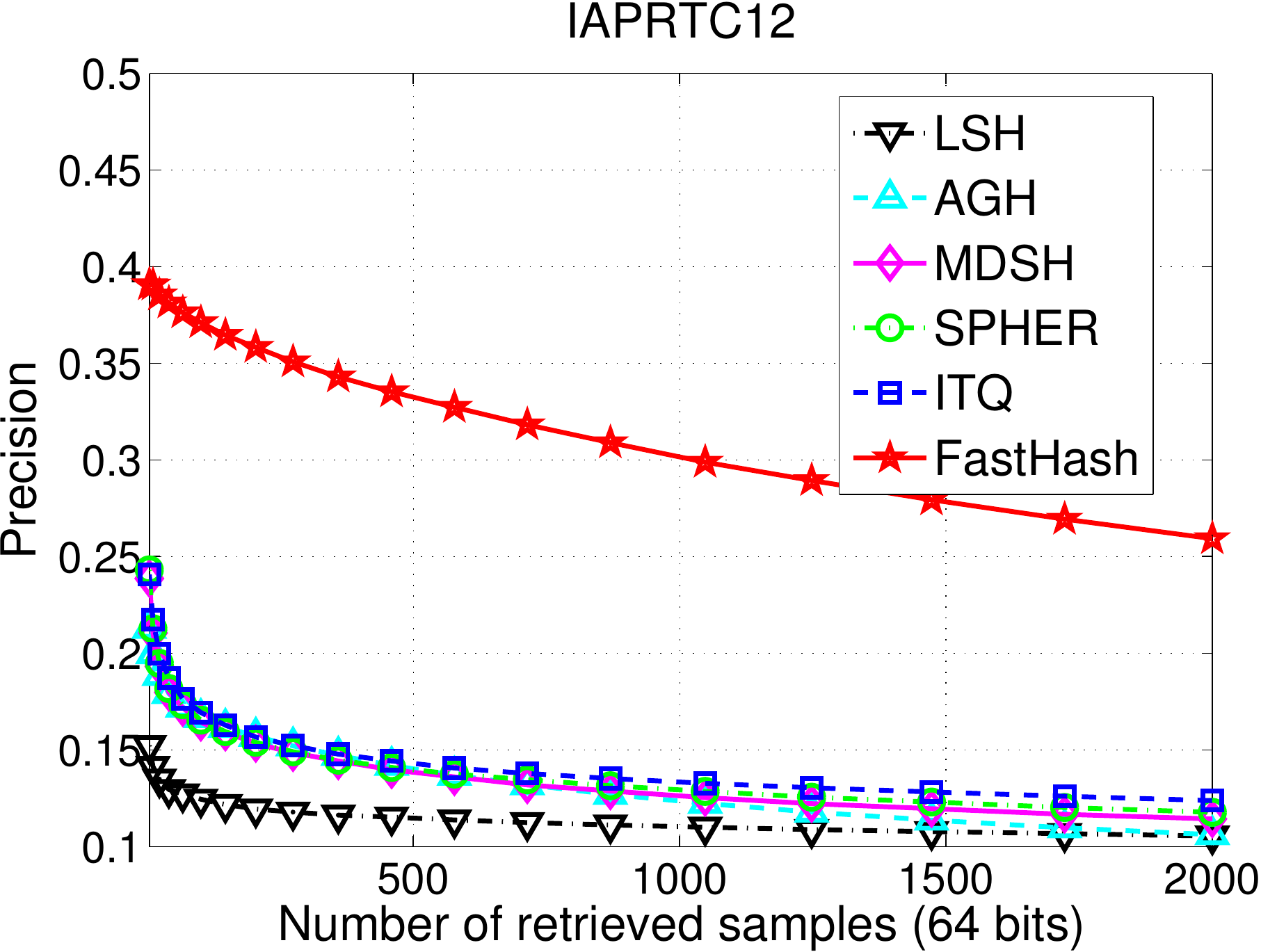}
   \includegraphics[width=.245\linewidth]{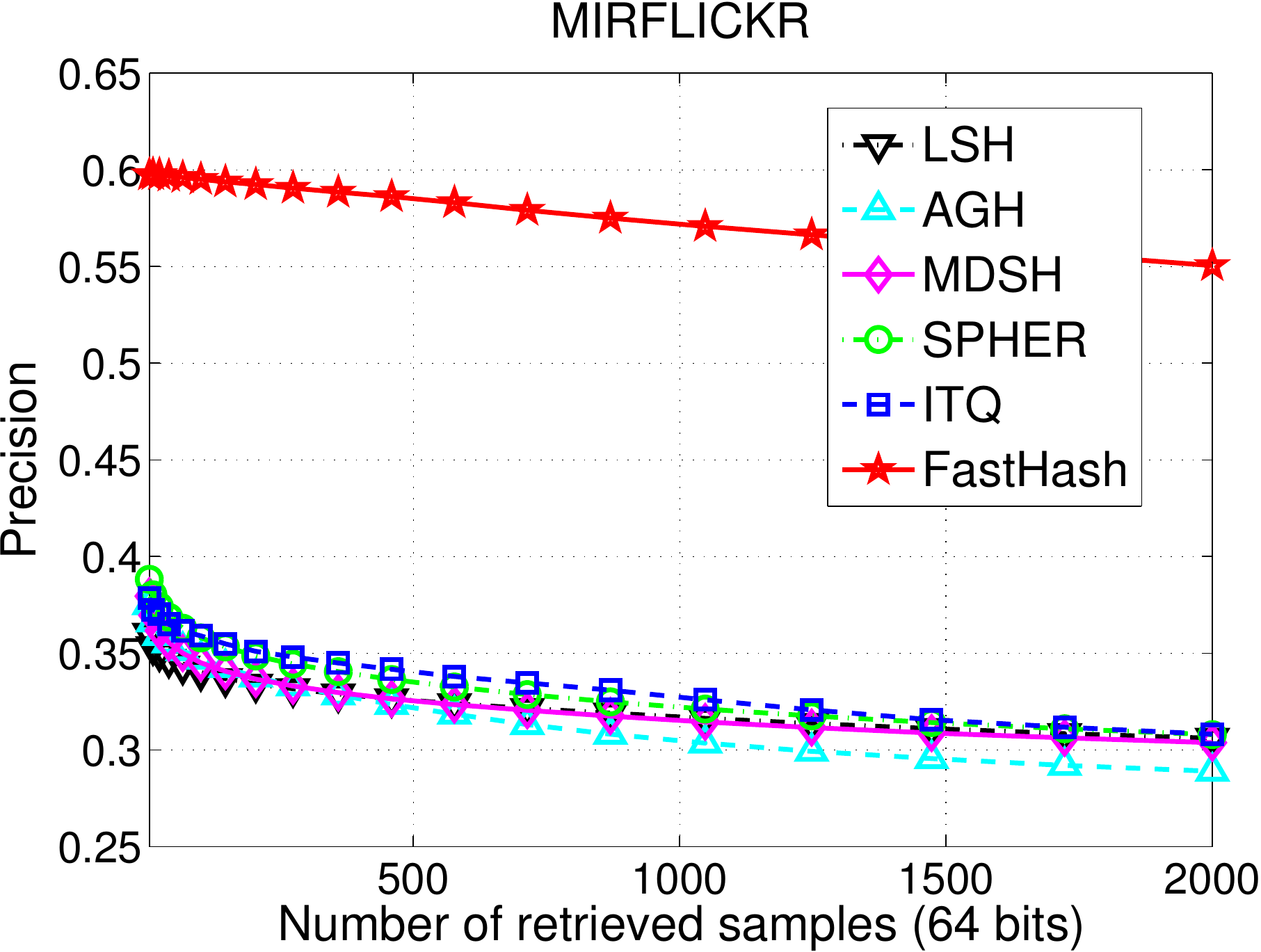}

    \caption{
    Comparison with a few unsupervised hashing methods.
    Unsupervised methods perform poorly for preserving label based similarity. Our \fasth performs significantly better.}
    \label{fh-fig:unsup}
\end{figure*}

\subsection{Comparison with unsupervised methods}
We compare to some popular unsupervised hashing methods:
LSH \cite{Gionis1999},
ITQ \cite{gong2012iterative},
Anchor Graph Hashing
(AGH) \cite{liu2011hashingGraphs}, Spherical Hashing (SPHER) \cite{jae2012},
Multi-dimension Spectral Hashing (MDSH) \cite{MDSH} \cite{MDSH}.
The retrieval performance is shown in Figure~\ref{fh-fig:unsup}.
Unsupervised methods perform poorly at preserving label based similarity. Our \fasth significantly outperforms others.
\subsection{Large dataset: SUN397}
The SUN397 \cite{xiao2010sun} dataset contains more than $100,000$ scene images.
$8000$ images are randomly selected as test queries, while the remaining $100,417$ images form the training set.
$11200$-dimensional codebook features are used here.
We compare with a number of supervised and unsupervised methods.
The depth for decision trees is set to $6$.
Results are presented in Table \ref{fh-tab:sun}
Supervised methods: KSH, BREs, SPLH and STHs are trained on a subset of $10$K examples.
Even on this sampled training set, the training of these methods are already impractically slow.
In contrast, our method can be efficiently trained with a long bit length ($1024$ bits)
on the whole training set (more than $100,000$ training examples).
Our \fasth significantly outperforms other methods.
The retrieval performance is also plotted in Figure~\ref{fh-fig:sun}.
It shows the results of those comparison methods that are able to be trained to $1024$ bits on the whole training set.
In terms of memory usage, many comparison methods require a large amount of memory for large matrix multiplication. In contrast, the decision tree learning in our method only involves simple comparison operations on quantized feature data ($256$ bins), thus \fasth consumes less than $7$GB for training.

\begin{figure}[t]
    \centering

   \includegraphics[width=.7\linewidth]{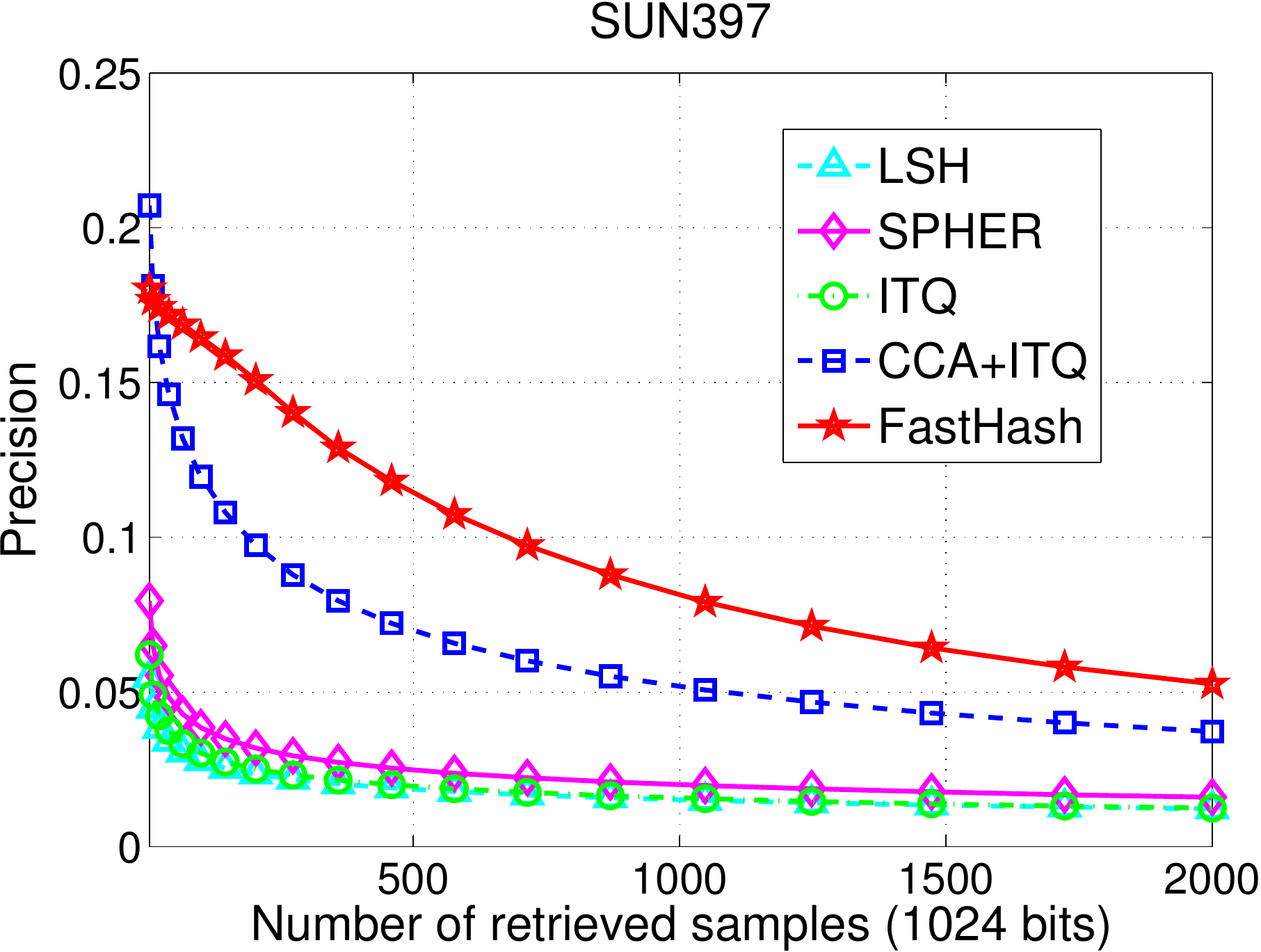}

    \caption{The top-$2000$ precision curve on large dataset SUN397 ($1024$ bits).
    Our \fasth performs the best.
    }
    \label{fh-fig:sun}
\end{figure}

\begin{table}[t]
\caption{Results on SUN397 dataset.
Our \fasth can be efficiently trained
 on this large training set.
\fasth significantly outperforms other methods.}
\centering
 \resizebox{1\linewidth}{!}
  {
  \begin{tabular}{ l l l | l l c c}
    \hline \hline
Method  &\#Train  &Bits &Train time &Test time  &Precision  &MAP
\\
\hline
\multicolumn{7}{ c}{SUN397} \\ \hline
KSH &10000  &64 &57045  &463  &0.034  &0.023  \\
BREs  &10000  &64 &105240 &23 &0.019  &0.013  \\
SPLH  &10000  &64 &27552  &14 &0.022  &0.015  \\
STHs  &10000  &64 &22914  &14 &0.010  &0.008  \\
\hline
ITQ  &100417 &1024 &1686 &127  &0.030  &0.021 \\
SPHER  &100417 &1024 &35954  &121  &0.039  &0.024 \\
LSH  &$-$ &1024 &$-$ &99 &0.028  &0.019 \\
CCA+ITQ &100417 &512  &7484 &66 &0.113  &0.076  \\
CCA+ITQ  &100417 &1024 &15580  &127  &0.120  &0.081 \\
\best FastH & 100417 &512  &29624  &302  &0.149  &0.142 \\
\best FastH  &100417 &1024 &62076  &536  &\best 0.165  &\best 0.163 \\
\hline\hline
  \end{tabular}
  }
\label{fh-tab:sun}
\end{table}

\subsection{Large dataset: ImageNet}

The large dataset ILSVRC2012 contains more than $1.2$ million images from ImageNet \cite{imagenet}.
We use the provided training set as the database (around $1.2$ million)
and the validation set as test queries ($50$K images).
Convolution neural networks (CNNs) have shown the best classification performance on this dataset \cite{krizhevsky2012imagenet}. As described in \cite{Donahue14},
the neuron activation values of internal layers of CNNs can be used as features.
By using the Caffe toolbox \cite{Jia13caffe} which implements the CNN architecture in \cite{krizhevsky2012imagenet},
we extract $4096$-dimensional features from the the seventh layer of the CNN.
We compare with a number of supervised and unsupervised methods.
The depth for the decision trees is set to $16$.
The smallest $2\%$ of data weightings are trimmed for decision tree learning.
Most comparing supervised methods become intractable on the full training set ($1.2$ million examples).
In contrast, our method is still able to be efficiently trained on the whole training set.
For comparison,
we also construct a smaller dataset (denoted as ImageNet-50) by sampling $50$ classes from ILSVRC2012.
It contains $25,000$ training images (500 images for each class) and $2500$ testing images.
Results of ImageNet-50 and the full ILSVRC2012
 are presented in Table \ref{fh-tab:imagenet}.
Our \fasth performs significantly better than others.
The retrieval performance of $128$ bits on the full ILSVRC2012 is plotted in Figure~\ref{fh-fig:imagenet}.

\begin{table}[t]
\caption{Results on two ImageNet datasets using CNN features.
ImageNet-50 is a small subset of ILSVRC2012.
Our \fasth significantly outperforms others.
}
\centering
 \resizebox{.9\linewidth}{!}
  {
  \begin{tabular}{ l l l | c c c}
    \hline \hline
Method  &\#Train  &Bits  &Precision  &MAP  &Prec-Recall
\\
\hline
\multicolumn{6}{ c}{ImageNet-50} \\ \hdashline
KSH	&25000	&64	&0.572	&0.460	&0.328\\
BREs	&25000	&64	&0.377	&0.246	&0.189\\
SPLH	&25000	&64	&0.411	&0.303	&0.217\\
STHs	&25000	&64	&0.625	&0.580	&0.412\\
ITQ+CCA	&25000	&64	&0.690	&0.668	&0.517\\
ITQ	&25000	&64	&0.492	&0.358	&0.266\\
SPHER	&25000	&64	&0.345	&0.210	&0.155\\
LSH	&$-$	&64	&0.064	&0.046	&0.023\\
\best FastHash	&25000	&64	&\best 0.697	&\best 0.718	&\best 0.532\\
\hline
\multicolumn{6}{ c}{ILSVRC2012} \\ \hdashline
CCA+ITQ	&1.2M	&64	&0.195	&0.133	&0.049\\
CCA+ITQ	&1.2M	&128	&0.289	&0.199	&0.090\\
CCA+ITQ &1.2M &1024 &0.428  &0.305  &0.160\\
ITQ	&1.2M	&64	&0.227	&0.132	&0.053\\
ITQ	&1.2M	&128	&0.294	&0.175	&0.080\\
ITQ &1.2M &1024 &0.368  &0.227  &0.108\\
LSH &$-$ &1024 &0.126  &0.065  &0.023\\
\best FastH	&1.2M	&64	&0.383	&0.301	&0.107\\
\best FastH	&1.2M	&128	&\best 0.458	&\best 0.390	&\best 0.171\\
\hline\hline
  \end{tabular}
  }
\label{fh-tab:imagenet}
\end{table}

\begin{figure}[t]
    \centering

   \includegraphics[width=.68\linewidth]{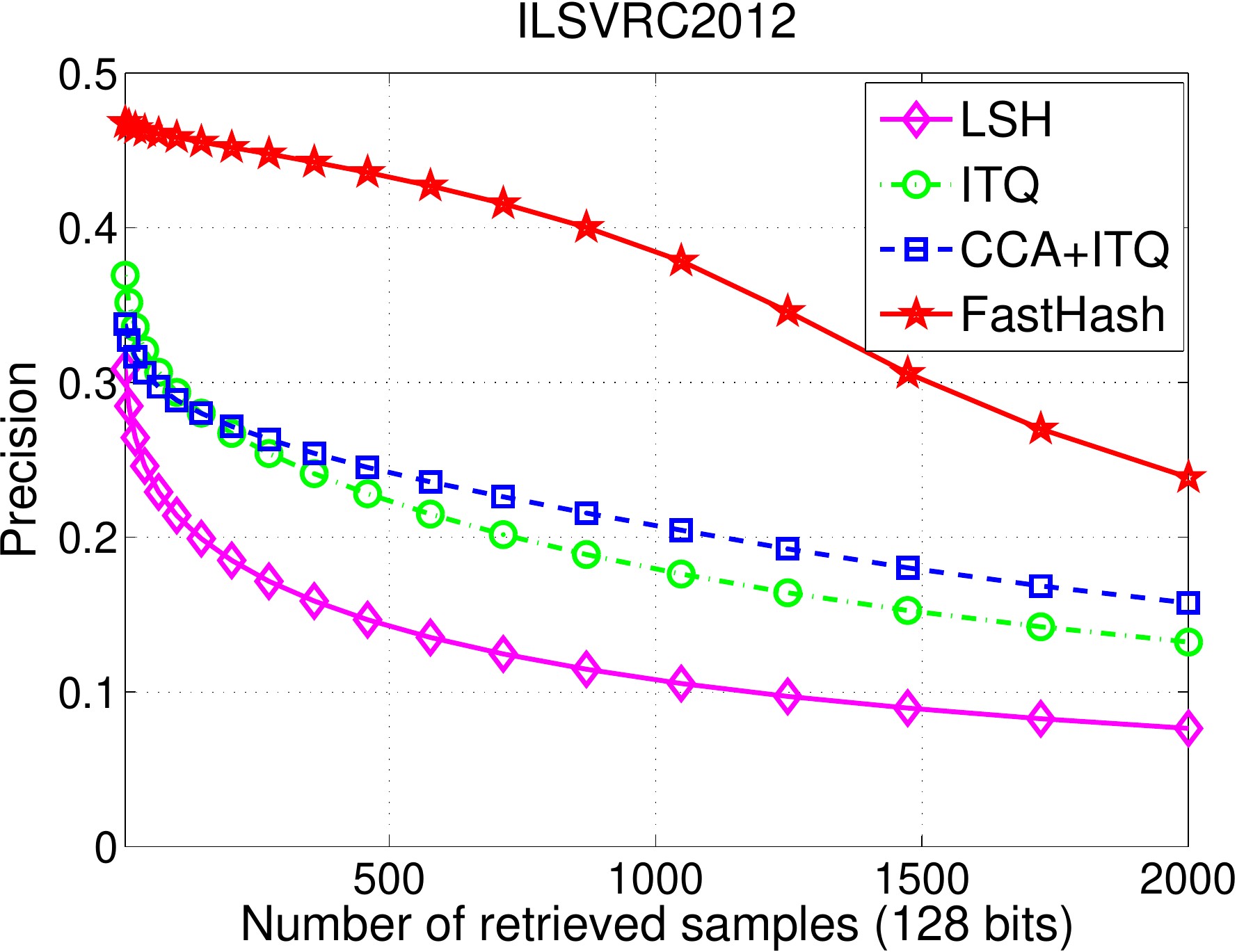}

    \caption{The top-$2000$ precision curve
     on large dataset ILSVRC2012
      ($128$ bits).
     Our \fasth outperforms others.
    }
    \label{fh-fig:imagenet}
\end{figure}

\subsection{Image classification}

Since binary codes have very small storage cost or network transfer cost,
image features can be compressed to binary codes by apply hashing methods.
Here we evaluate the image classification performance of using binary codes as features on the dataset ILSVRC2012.
Hashing methods are trained on the CNN features.
We apply two types of classification methods: the K nearest neighbor (KNN) classifier and the one-vs-all linear SVM classifier. KNN classification is performed by majority voting of top-K retrieved neighbors with smallest hamming distances. Results are shown in Table \ref{fh-tab:imagenet-mc}.
Our method outperforms all comparison hashing methods.

The CNN features used here
are extracted on the center crops of images using Caffe \cite{Jia13caffe}.
We also report the results of CNN methods which have the state-of-the-art results of this dataset.
As shown in Table \ref{fh-tab:imagenet-mc},
the performance gap is around $8\%$ between the error rate of our hashing method and that of Caffe with similar settings (only using center crops). However, $128$-bit binary codes in our methods take up around $1000$ times less storage than the CNN features with $4096$-dimensional float values. It shows that our method is able to perform effective binary compression without large performance loss.

\begin{table}[t]
\caption{Image classification results on dataset ILSVRC2012.
Binary codes are generated as features for training classifiers.
Our \fasth outperforms other hashing methods for binary compression of features.
}
\centering
\resizebox{0.9\linewidth}{!}
  {
  \begin{tabular}{ l l | c c}
    \hline \hline
Hashing method	&bits	&KNN-50 test error	&1-vs-all SVM test error\\
\hline
\multicolumn{4}{ c}{ILSVRC2012} \\ \hdashline
LSH	&128	&0.594	&0.939\\
ITQ	&128	&0.557	&0.919\\
CCA+ITQ	&64	&0.716	&0.691\\
CCA+ITQ	&128	&0.614	&0.583\\
FastHash	&64	&0.572	&0.567\\
FastHash	&128	&\best 0.516	&\best 0.512\\
\hline\hline
\multicolumn{2}{ l|}{Classification method} 	& \multicolumn{2}{ c}{Test error} \\ \hline
\multicolumn{2}{ l|}{Caffe \cite{Jia13caffe} (center crop) } & \multicolumn{2}{ c}{0.433} \\
\multicolumn{2}{ l|}{Caffe \cite{Jia13caffe} } & \multicolumn{2}{ c}{0.413} \\
\multicolumn{2}{ l|}{CNNs \cite{krizhevsky2012imagenet} (one model) } & \multicolumn{2}{ c}{0.407} \\
  \hline \hline
  \end{tabular}
  }
\label{fh-tab:imagenet-mc}
\end{table}

\section{Conclusion}

We have shown that various kinds of loss functions
and hash functions can be placed
in a unified learning framework for supervised hashing.
By using the proposed binary inference algorithm Block GraphCut
and learning decision tree hash functions,
our method can be efficiently trained on large-scale and high-dimensional data
and achieves high testing precision,
which indicates its practical significance on
many applications like large-scale image retrieval.

\section*{Acknowledgements}

This research was in part supported by the Data to Decisions Cooperative Research Centre.
C. Shen's participation was in part support by ARC Future Fellowship.

{
\bibliographystyle{IEEEtran}
\bibliography{Bibliography}
}

 \begin{IEEEbiography}[{\includegraphics[width=1in,height=1.25in,clip,keepaspectratio]{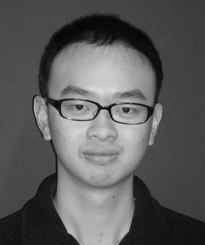}}]{Guosheng Lin}
 is a Research Fellow at School of Computer Science, The University of Adelaide.
 He completed his PhD degree at the same university in 2014.
 His research interests are on computer vision and machine learning.
 He received a Bachelor degree and a Master degree from the
 South China University of Technology in computer science in 2007 and 2010 respectively.
 \end{IEEEbiography}

\begin{IEEEbiography}[{\includegraphics[width=1in,height=1.25in,clip,keepaspectratio]
{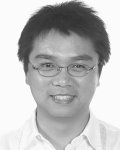}}]{Chunhua Shen}
is a Professor at School of Computer Science, The University of Adelaide.
His research interests are in the
intersection of computer vision and statistical machine learning.
He studied at Nanjing University, at Australian National University,
and received his PhD degree from University of Adelaide. In 2012,
he was awarded the Australian Research Council Future Fellowship.
\end{IEEEbiography}

 \begin{IEEEbiography}[{\includegraphics[width=1in,height=1.25in,clip,keepaspectratio]{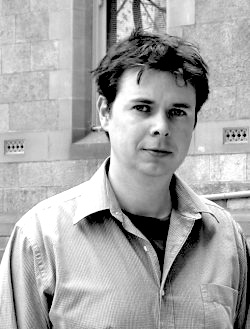}}]{Anton van~den~Hengel}
 is the founding Director of The Australian Centre for Visual Technologies (ACVT).
 He is a Professor in  School of Computer Science,
 The University of Adelaide. He received a PhD in Computer Vision in 2000,
 a Master Degree in Computer Science in 1994,
 a Bachelor of Laws in 1993, and a Bachelor of Mathematical Science in 1991, all from
 The University of Adelaide.
 \end{IEEEbiography}
\end{document}